\documentclass[twoside]{article}

\usepackage[preprint]{neutral-preprint}
\usepackage[round]{natbib}
\renewcommand{\bibname}{References}
\renewcommand{\bibsection}{\subsubsection*{\bibname}}

\usepackage{url}
\usepackage{bbm}
\usepackage{graphicx}
\usepackage{algorithm}
\usepackage{algpseudocode}
\usepackage{amsmath}
\usepackage{amssymb}
\usepackage{booktabs}
\usepackage{tabularx}
\usepackage{multirow}
\usepackage{array}
\usepackage{graphicx}
\usepackage{subcaption}
\usepackage{xcolor}
\usepackage{tikz}
\usetikzlibrary{arrows.meta,calc,shapes.geometric}
\definecolor{posgreen}{RGB}{44,160,44}
\definecolor{negred}{RGB}{214,39,40}
\definecolor{actorange}{RGB}{255,127,14}
\tikzset{
  ov/.style={
    every node/.style={font=\footnotesize},
    ovst/.style={draw=gray!70, fill=gray!12, rounded corners=1.2pt,
                 minimum width=0.36cm, minimum height=0.36cm, inner sep=0pt},
    ovcur/.style={ovst, draw=black, fill=gray!35, line width=0.8pt},
    ovpos/.style={ovst, draw=posgreen, fill=posgreen!18, line width=0.9pt},
    ovneg/.style={ovst, draw=negred, fill=negred!15, line width=0.9pt},
    ovfaded/.style={ovst, draw=gray!40, fill=gray!5},
    ovroll/.style={ovst, draw=black!70, fill=white, dashed, line width=0.6pt},
    ovtr/.style={->, gray!70, line width=0.5pt},
    ovbox/.style={draw, rounded corners=2.5pt, line width=0.8pt, align=center},
    ovenc/.style={draw, rounded corners=1.5pt, line width=0.7pt,
                  minimum width=1.5cm, minimum height=0.36cm},
    ovnote/.style={font=\scriptsize, text=gray, align=center},
  }
}
\usepackage{wrapfig}
\usepackage{amsthm}
\usepackage{comment}
\newtheorem{theorem}{Theorem}[section]
\newtheorem{lemma}[theorem]{Lemma}
\usepackage{amsthm}
\newtheorem{proposition}{Proposition}

\theoremstyle{definition}
\newtheorem{definition}{Definition}
\newtheorem{examplex}{Example}
\newenvironment{example}{\pushQED{\qed}\begin{examplex}}{\popQED\end{examplex}}

\begin{document}

%

%

\twocolumn[

\aistatstitle{Counterfactual Rollouts for Offline Contrastive Goal-Conditioned Reinforcement Learning}

\begin{center}
Tingrui Huang$^{1}$ \qquad
Ai Bo$^{2}$ \qquad
Junzhe Zhang$^{2}$

\vspace{0.5em}

{\small
$^{1}$Eindhoven University of Technology,
Eindhoven, The Netherlands
}

{\small
$^{2}$Department of Electrical Engineering and Computer Science,\\
Syracuse University, Syracuse, NY, USA
}
\end{center}

]

\begin{abstract}
  Contrastive reinforcement learning (CRL) learns goal-conditioned values by associating state--action pairs with the states subsequently reached, assuming that no unobserved variable jointly affects the recorded action and its outcome. When offline demonstrations come from a demonstrator with information hidden from the learner, this assumption fails and the critic absorbs spurious action--outcome associations. We formalize this setting with confounded Markov decision processes and introduce counterfactual transitions that separate what the data identify from what they do not; minimizing over all compatible transitions yields a worst-case occupancy that lower-bounds the true one. Causal Contrastive Reinforcement Learning (CauCRL) learns a pessimistic counterfactual transition from positive and negative offline demonstrations and trains the critic on counterfactual rollouts, substantially improving success rates over standard CRL on confounded goal-reaching tasks.
\end{abstract}

\section{Introduction}

\begin{figure*}[t]
\centering
\begin{tikzpicture}[ov]
\foreach \i/\x in {1/0.75,2/1.35,3/1.95,4/2.55}{
  \node[ovpos] (p\i) at (\x,2.3) {};
  \node[ovpos] (n\i) at (\x,0.7) {};
}
\node[ovneg] at (n4) {};
\node[text=negred, font=\scriptsize\bfseries] at (n4) {$\times$};
\foreach \i/\j in {1/2,2/3,3/4}{ \draw[ovtr] (p\i) -- (p\j); }
\node[anchor=east] at (p1.west) {$\mathcal D^{+}$};
\node[anchor=east] at (n1.west) {$\mathcal D^{-}$};
\node[ovnote, anchor=south] at (1.65,2.5) {positive demonstrations};
\node[draw, dashed, circle, inner sep=0.5pt, minimum size=0.38cm] (U) at (1.65,1.5) {$U$};
\foreach \x in {1.05,1.65,2.25}{ \draw[->, dashed, line width=0.5pt] (U) -- (\x,2.28); }
\foreach \x in {1.05,1.65,2.25}{ \draw[->, dashed, line width=0.5pt] (U) -- (\x,0.72); }
\node[ovnote, anchor=north] at (1.65,0.5) {negative demonstrations};
\node[ovbox, minimum width=2.2cm] (T) at (5.0,1.5)
  {{\scriptsize pessimistic counterfactual}\\[-2pt]{\scriptsize transition}\\[-1pt]$\mathcal T_\eta(s,a,a',s')$};
\draw[->, line width=0.7pt, posgreen!80!black]
  (p4.east) .. controls (3.1,2.3) and (3.0,1.62) .. ($(T.west)+(0,0.12)$)
  node[pos=0.35, above right=-2pt] {$\mathcal L^{+}$};
\draw[->, line width=0.7pt, negred]
  (n4.east) .. controls (3.1,0.7) and (3.0,1.38) .. ($(T.west)+(0,-0.12)$)
  node[pos=0.35, below right=-2pt] {$\mathcal L^{-}$};
\node[ovcur] (s) at (7.35,1.5) {};
\node[anchor=north east, inner sep=1pt] at (s.south) {$s$};
\draw[->, line width=0.7pt] (T.east) -- (s.west);
\foreach \i/\x in {1/8.1,2/8.85,3/9.6}{ \node[ovroll] (r\i) at (\x,1.5) {}; }
\node[ovpos] (rp) at (10.35,1.5) {};
\draw[->, actorange, line width=1pt] (s) -- (r1) node[midway, above=-1pt] {$a$};
\draw[->, line width=0.5pt] (r1) -- (r2);
\draw[->, line width=0.5pt] (r2) -- (r3);
\draw[->, line width=0.5pt] (r3) -- (rp);
\node[ovnote, anchor=south] at (9.2,1.72) {counterfactual rollout};
\foreach \i/\x in {1/8.1,2/8.85,3/9.6}{ \node[ovfaded] (o\i) at (\x,0.7) {}; }
\node[ovfaded] (op) at (10.35,0.7) {};
\node[text=negred, font=\scriptsize\bfseries] at (op) {$\times$};
\draw[ovtr, gray!40] (s.south east) -- (o1);
\foreach \i/\j in {1/2,2/3}{ \draw[ovtr, gray!40] (o\i) -- (o\j); }
\draw[ovtr, gray!40] (o3) -- (op);
\node[anchor=west, text=gray, inner sep=1.5pt] at (op.east) {$s^{+}$};
\node[ovnote, anchor=north] at (9.2,0.5) {observed future (confounded)};
\node[ovenc] (phi) at (12.4,2.45) {};
\node[ovenc] (psi) at (12.4,1.5) {};
\foreach \k in {-2,...,2}{
  \draw ($(phi)+(\k*0.24,0)$) circle (0.07);
  \draw ($(psi)+(\k*0.24,0)$) circle (0.07);
}
\node[anchor=south, inner sep=1.5pt] at (phi.north) {$\varphi(s,a)$};
\node[anchor=south, inner sep=1.5pt] at (psi.north) {$\psi(\cdot)$};
\node (dot) at (13.75,1.975) {$\times$};
\draw[->, line width=0.6pt] (phi.east) -- (dot);
\draw[->, line width=0.6pt] (psi.east) -- (dot);
\node[anchor=west] at (dot.east) {$\mathcal L(s,a,\tilde s^{+},s^{-})$};
\draw[->, line width=0.7pt, posgreen!80!black] (rp.east) -- (psi.west)
  node[pos=0.7, above=-1pt] {$\tilde s^{+}$};
\draw[->, line width=0.7pt, posgreen!80!black] (11.0,1.5) |- (phi.west)
  node[pos=0.75, above=-1pt] {$(s,a)$};
\fill[posgreen!80!black] (11.0,1.5) circle (1.3pt);
\node[ovneg] (neg) at (12.4,0.55) {};
\node[anchor=west, text=negred, inner sep=2pt, font=\scriptsize] at (neg.east) {negative sample $s^{-}\sim P^{-}(s)$};
\draw[->, dashed, line width=0.6pt, negred] (neg.north) -- (psi.south);
\end{tikzpicture}
\caption{
    Overview of CauCRL. A pessimistic counterfactual transition
    $\mathcal T_\eta$ is learned from positive demonstrations
    $\mathcal D^{+}$ and negative demonstrations
    $\mathcal D^{-}$, confounded by $U$. Its rollout from $(s,a)$ yields a
    counterfactual future $\tilde s^{+}$ that replaces the observed future
    $s^{+}$ as the positive sample, contrasted against a negative sample
    $s^{-}\sim P^{-}(s)$.
}
\label{fig:overview}
\end{figure*}
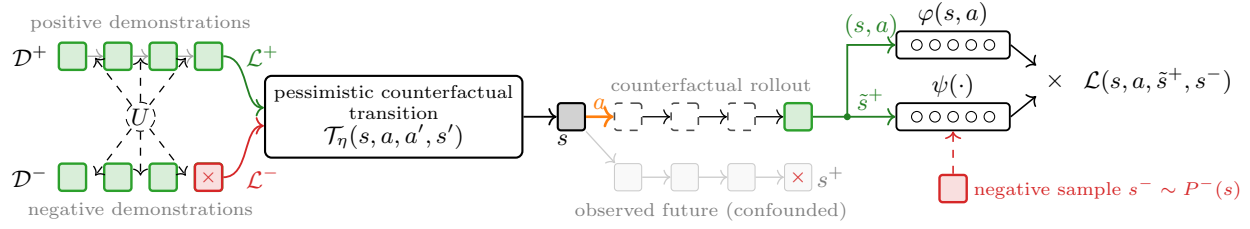

Learning to make good decisions from sparse feedback is a long-standing challenge in reinforcement learning (RL) \citep{sutton2018reinforcement}. In many real-world applications, the agent receives little or no reward signal: it may only observe expert demonstrations, or be told whether a task has been accomplished. Among various paradigms of learning from sparse feedback, imitation learning \citep{pomerleau1989alvinn,ross2011reduction,ho2016generative}, which learns from demonstrations without any reward signal, and goal-conditioned RL \citep{kaelbling1993ijcai-learning,pmlr-v37-schaul15}, which learns to reach goals specified at test time, have emerged as crucial and popular directions. Goal-conditioned RL is particularly appealing, since a single learned policy can be reused across many tasks \citep{duan2016rl2}.

A variety of algorithms have been proposed for goal-conditioned RL, including reward shaping \citep{ng1999policy}, hindsight relabeling \citep{andrychowicz2017hindsight}, goal-conditioned supervised learning \citep{ghosh2021learning}, and universal value function approximation \citep{pmlr-v37-schaul15}, to name a few; see \citet{liu2022goal} for a survey. Among them, contrastive RL is a simple and effective framework that trains a critic to distinguish future states of the same trajectory from states of unrelated trajectories, and uses this critic as a goal-conditioned value function \citep{DBLP:journals/corr/abs-2011-08909,NEURIPS2022_e7663e97}. While successful in practice, these methods rely on the crucial assumption that the actions observed in the offline data, along with the subsequent states they produce, are not simultaneously influenced by unobserved confounders. If this assumption is violated, the probability of reaching the goal following a candidate policy becomes non-identifiable, i.e., it cannot be uniquely determined from the available data and model assumptions. The following example illustrates such challenges.

\begin{example}[Challenges of Unobserved Confounding]
\label{ex:windy}
Consider the windy-corridor task described in Fig.~\ref{fig:minigrid-result} (left), where the agent must reach the goal without touching the lava \citep{bo2026counterfactual}. A short route passes through a corridor next to the lava, while a longer route is safe. At each time step $t$, a wind $U^{(t)}$ blows through the corridor; when it is unfavorable, moving forward along the short route pushes the agent into the lava. The learner cannot observe the wind. Instead, it observes demonstrations from an expert who can: the expert moves forward along the short route when the wind is favorable and pauses otherwise, while a small fraction of demonstrations follow the longer route. The wind $U^{(t)}$ thus becomes an unobserved confounder affecting both the observed action and the next state. We apply standard contrastive RL \citep{NEURIPS2022_e7663e97} to these demonstrations. As shown in Fig.~\ref{fig:minigrid-result} (right), it fails to achieve a high success rate: since forward moves along the short route are almost always followed by reaching the goal in the demonstrations, the critic overvalues them, even though the learner cannot determine when they are safe.
\end{example}

\begin{figure}[t]
    \centering
    \includegraphics[width=0.49\columnwidth]{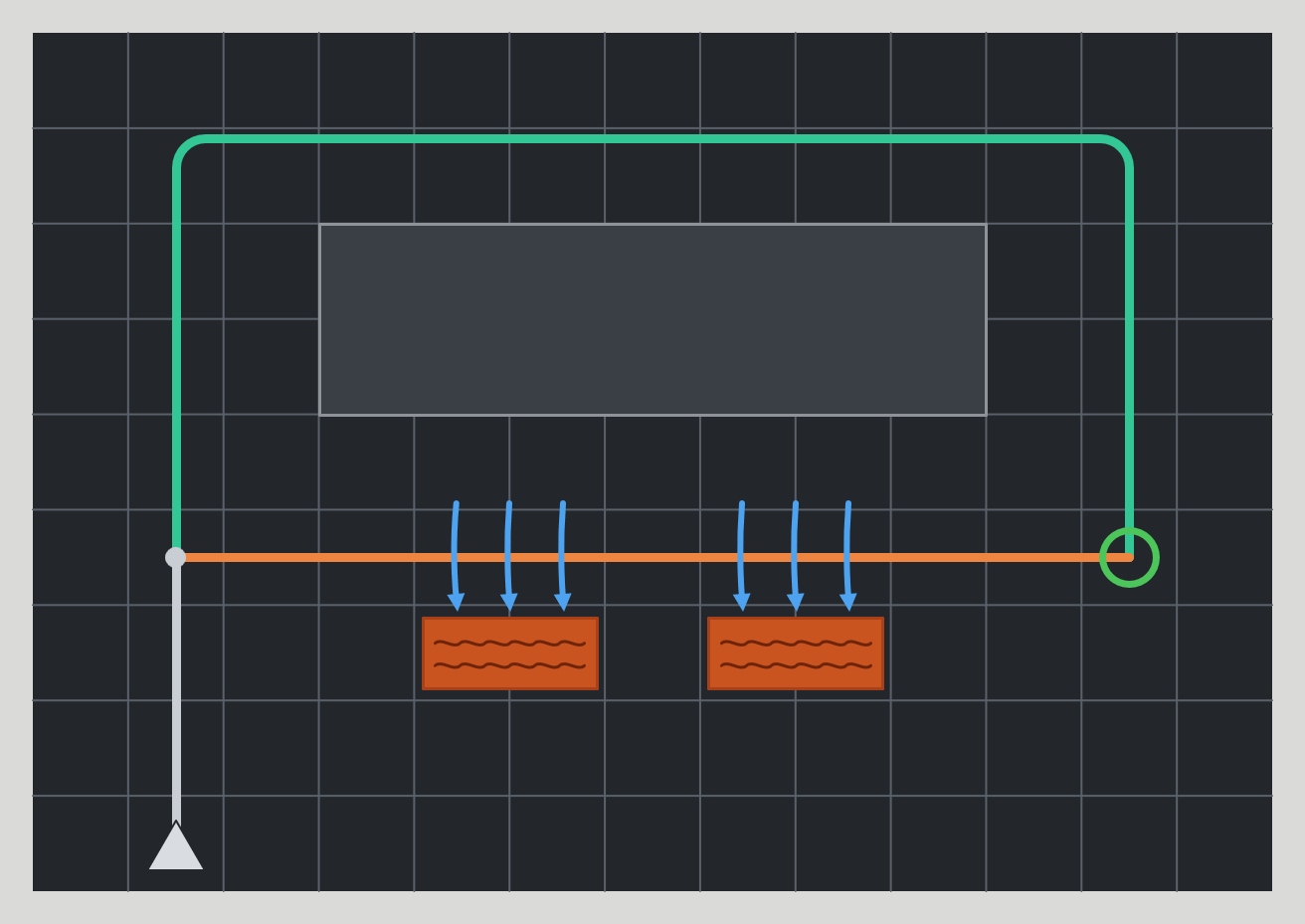}
    \hfill
    \includegraphics[width=0.49\columnwidth]{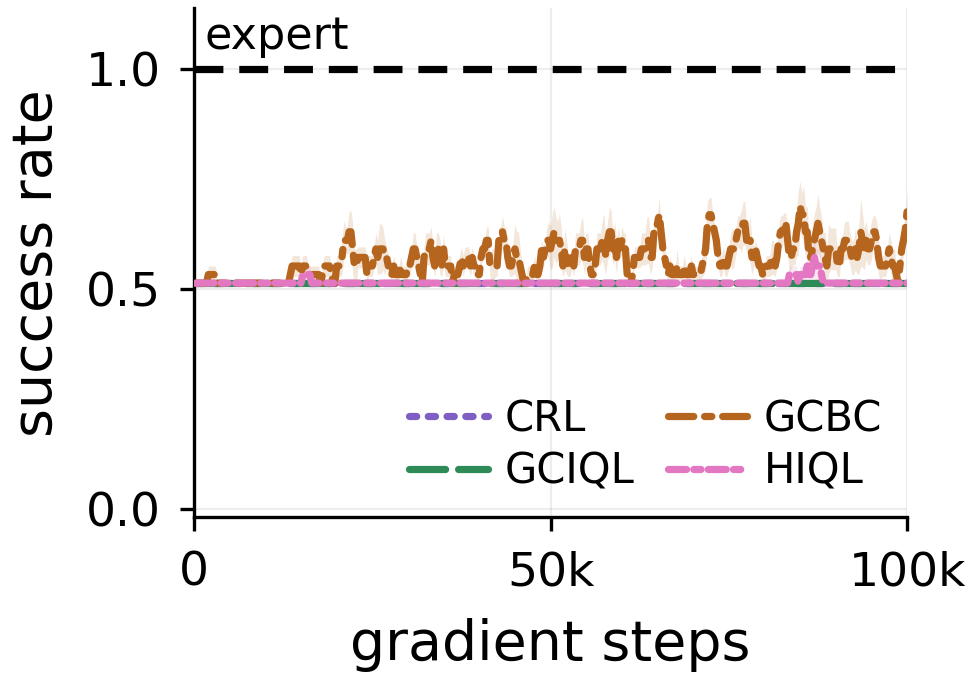}
\caption{
Windy corridor (left) and success rate over training (right):
under confounded demonstrations, all baselines stay below the expert success rate of $1.0$.
}
    \label{fig:minigrid-result}
\end{figure}

Recently, a growing body of literature has explored the nuanced interactions between causal inference \citep{pearl2009causality,bareinboim2022pearlhierarchy} and RL to address data biases in decision-making under uncertainty, known as \emph{causal reinforcement learning}. Several algorithms have been proposed for various policy learning settings, including off-policy evaluation and learning \citep{kallus2019confoundingrobustpolicyimprovement,NEURIPS2020_da21bae8,kallus2020confounding,pmlr-v130-bennett21a,pmlr-v286-zhang25d}, offline-to-online learning \citep{NEURIPS2019_8252831b}, robust offline and deep RL \citep{joshi2024aaai-safe,NEURIPS2025_c9a177f0,li2026cfql}, meta-RL \citep{bo2026counterfactual}, and imitation learning \citep{dehaan2019causal,zhang2020causal,ruan2024causal,balazadeh2025sequentialdecisionmakingexpert}, to name a few. Despite these progresses, a systematic approach for goal-conditioned RL under unobserved confounding is still missing. It is unclear how one can learn the goal-reaching occupancy underlying contrastive RL when the demonstrations that provide its positive samples are confounded.

This paper aims to address this gap by proposing Causal Contrastive Reinforcement Learning (CauCRL; Fig.~\ref{fig:overview}), which trains contrastive critics on pessimistic counterfactual rollouts that remain consistent with the confounded demonstrations. More specifically, our contributions are summarized as follows. (1) We formulate goal-conditioned RL from confounded demonstrations and introduce counterfactual transitions together with a worst-case occupancy measure that lower-bounds the true probability of reaching the goal (Sec.~\ref{sec:counterfactual-modeling}); (2) We develop a computational procedure for CauCRL that learns a pessimistic counterfactual transition from positive and negative demonstrations and uses its counterfactual rollouts as positive samples for contrastive critic learning (Sec.~\ref{sec:method}); (3) Finally, we validate CauCRL through simulations on goal-reaching tasks with confounded demonstrations, ranging from discrete navigation to high-dimensional continuous control (Sec.~\ref{sec:experiments}). Due to space constraints, all proofs are provided in Appendix~\ref{app:theory}. 

\textbf{Notations.}
We use capital letters to denote random variables ($X$), lowercase letters for their values ($x$), and calligraphic letters for their domains ($\mathcal{X}$).
For a set $\mathcal{X}$, $|\mathcal{X}|$ denotes its cardinality.
We write $P(X)$ for a probability distribution over $X$ and $P(Y\mid X)$ for the set of conditional distributions $P(Y\mid X=x)$ over all realizations $x$.
We abbreviate $P(X=x)$ as $P(x)$ and $P(Y=y\mid X=x)$ as $P(y\mid x)$.
Finally, $\mathbbm{1}\{Z=z\}$ is an indicator that equals $1$ if $Z=z$ and $0$ otherwise.

\section{Goal-Conditioned RL with Unobserved Confounding}
\label{sec:problem}

We consider goal-conditioned RL \citep{NEURIPS2022_e7663e97} from
offline trajectories of a demonstrator that may use information
unavailable to the learner, modeled by a confounded Markov decision
process \citep{zhang2016mdpconfounders,bo2026counterfactual}.

\begin{definition}
\label{def:cmdp}
A confounded goal-conditioned Markov decision process is a tuple
$\mathcal{M}=\langle\mathcal{S},\mathcal{A},\mathcal{Y},\mathcal{U},\mathcal{F},P\rangle$,
where (1) $\mathcal{S}$, $\mathcal{A}$, and $\mathcal{Y}=\{0,1\}$
are, respectively, the spaces of observed states, actions, and
rewards; (2) $\mathcal{U}$ is the space of unobserved latent
variables; (3) $\mathcal{F}$ is a set of structural mechanisms
$f_S:\mathcal{S}\times\mathcal{A}\times\mathcal{U}\rightarrow\mathcal{S}$
and $f_A:\mathcal{S}\times\mathcal{U}\rightarrow\mathcal{A}$,
where $f_S$ determines the state transition and $f_A$ determines the
behavioral policy of the demonstrator; and (4) $P$ is an exogenous
distribution over $\mathcal{U}$.
Goals are states, $s^*\in\mathcal{S}$, and the reward indicates
whether the goal is reached, $Y=\mathbbm{1}\{S=s^*\}$.
\end{definition}

The process starts from an initial state $S^{(0)}\sim P(S^{(0)})$. At
each time step $t$, a latent variable $U^{(t)}$ is drawn independently
from $P(U)$, an action $A^{(t)}$ is taken, and the environment moves to
$S^{(t+1)}\leftarrow f_S(S^{(t)},A^{(t)},U^{(t)})$ with reward $Y^{(t)}=\mathbbm{1}\{S^{(t)}=s^*\}$. Since the structural functions and exogenous distributions are invariant with regard to time steps, we write $(S,A,S')$ for $(S^{(t)},A^{(t)},S^{(t+1)})$ when the time
index is clear.

The learner acts through a policy $\pi(a\mid s)$. Following a policy
is an intervention $do(\pi)$ that replaces the behavior mechanism $f_A$
by $\pi$ while leaving the remaining mechanisms unchanged
\citep{pearl2009causality}; similarly, $do(A^{(t)}=a)$ sets a single
action to $a$. A subscript on a random variable denotes the
corresponding counterfactual: $S'_a$ is the next state had the action
been set to $a$. The resulting interventional transition is
$P_a(s'\mid s)=P(S'_a=s'\mid S=s)=\int_{\mathcal{U}}\mathbbm{1}\{f_S(s,a,u)=s'\}\,dP(u)$,
and a policy $\pi$ induces the interventional distribution over
trajectories $(\bar s,\bar a)=(s^{(0)},a^{(0)},s^{(1)},a^{(1)},\dots)$,
\[
\begin{aligned}
P_\pi(\bar s,\bar a)
&=
P(s^{(0)})
\prod_{t\geq0}
\pi(a^{(t)}\mid s^{(t)})
P_{a^{(t)}}(s^{(t+1)}\mid s^{(t)}).
\end{aligned}
\]
For a state--action pair $(s,a)$ and goal $s^*$, we define the
discounted goal-reaching occupancy as
\begin{equation}
\label{eq:goal-occupancy}
\begin{aligned}
d^\pi(s^*\mid s,a)
&=
(1-\gamma)
\sum_{k=0}^{\infty}
\gamma^k
\\[-2pt]
&\quad\times
P_{a,\pi}\!\left(
S^{(t+k)}=s^*
\mid
S^{(t)}=s
\right),
\end{aligned}
\end{equation}
where $P_{a,\pi}$ sets $A^{(t)}$ to $a$ and lets subsequent actions
follow $\pi$. Since $Y^{(t+k)}=\mathbbm{1}\{S^{(t+k)}=s^*\}$, this
occupancy coincides with the Q-function of $\pi$,
$d^\pi(s^*\mid s,a)=Q^\pi(s,a)
:=(1-\gamma)\,\mathbb{E}_{a,\pi}[\sum_{k=0}^{\infty}\gamma^kY^{(t+k)}\mid S^{(t)}=s]$
\citep{NEURIPS2022_e7663e97}. The learner seeks a policy
$\pi$ that maximizes this goal-reaching probability.

Contrastive RL \citep{NEURIPS2022_e7663e97} learns this occupancy with
a critic $f_\omega(s,a,s^*)=\varphi_\omega(s,a)^\top\psi_\omega(s^*)$ by
contrasting positive and negative samples. For a state--action pair
$(s,a)$, a positive sample $s^+$ is a state from the discounted future
of a trajectory under $\pi$, while a negative sample
$s^-\sim P^-(s)$ is drawn from a marginal state distribution that
differs from the state distribution of the current policy estimate
$\pi$. The critic is trained with the binary NCE objective
\begin{align}
\max_{\omega}\quad
\mathbb{E}_{\substack{(s,a,s^+)\sim P_\pi\\ s^-\sim P^-}}\Big[
&
\log \sigma\!\left(f_\omega(s,a,s^+)\right)
\label{eq:crl-nce}
\\[-8pt]
&+
\log\!\left(
1-\sigma\!\left(f_\omega(s,a,s^-)\right)
\right)
\Big].
\nonumber
\end{align}
At the optimum, the critic recovers the unnormalized occupancy of the
positive samples \citep{NEURIPS2022_e7663e97}:
$\exp f^*(s,a,s^*)=d^\pi(s^*\mid s,a)/P^-(s^*)$, i.e., $Q^\pi$ up to an
action-independent factor, so the policy can be improved by favoring
actions with high critic values.

In the offline setting, on the other hand, the learner cannot intervene and must instead
learn from recorded trajectories of a demonstrator, who observes the
latent variable and selects actions by the behavioral policy
$A^{(t)}\leftarrow f_A(S^{(t)},U^{(t)})$. Demonstrations therefore follow the observational distribution
\[
\begin{aligned}
P(\bar s,\bar a)
&=
P(s^{(0)})
\prod_{t\geq0}
\int_{\mathcal{U}}
\mathbbm{1}\{f_A(s^{(t)},u)=a^{(t)}\}
\\
&\quad\times
\mathbbm{1}\{f_S(s^{(t)},a^{(t)},u)=s^{(t+1)}\}\,dP(u),
\end{aligned}
\]
in which the same latent realization $u$ drives both the action and
the transition. From these data, the learner can estimate the
observational transition
$P(s'\mid s,a)=P(S'=s'\mid S=s,A=a)$, but not the latent variable. The
offline data consist of two parts. Positive demonstrations follow an
observational distribution $P^+(\bar s,\bar a)$ of this form, induced
by the expert behavioral policy $f_A$, and contain trajectories reaching
the goal. Negative demonstrations are generated by a non-expert
behavioral policy, such as a randomized one, but record the
eventual states it reaches, such as the lava in
Example~\ref{ex:windy}. We denote their state distribution by
$P^-(s)$, which serves as the negative distribution in
Eq.~\eqref{eq:crl-nce}.\footnote{Unlike the original CRL
\citep{NEURIPS2022_e7663e97}, negatives here are drawn from the
non-expert demonstrations $P^-$, which we expect to sufficiently cover
undesirable states for the task; Sec.~\ref{sec:method} uses them as an
inductive bias toward negative outcomes for counterfactual
transitions.}

Offline CRL trains Eq.~\eqref{eq:crl-nce} on these data, taking the
positive sample $(s,a,s^+)\sim P^+$ from the discounted future of the
same demonstration instead of from $P_\pi$. This change of data is
where the guarantee breaks: because $U$ affects both the demonstrator's
action and the next state, generally $P(s'\mid s,a)\neq P_a(s'\mid s)$,
so positive samples follow the observational counterpart
$d(s^*\mid s,a)$ of Eq.~\eqref{eq:goal-occupancy}, with $P_{a,\pi}$
replaced by $P(\cdot\mid S^{(t)}=s,A^{(t)}=a)$. The critic then
recovers $d(s^*\mid s,a)/P^-(s^*)$ rather than $Q^\pi$, biasing both
the critic and the policy improved from it. The following example
illustrates this challenge.

\begin{example}[Example~\ref{ex:windy} continued]
\label{ex:windy-cont}
In the windy corridor, $U^{(t)}$ is the wind, which the demonstrator
observes but the learner does not. The expert policy $f_A$ moves
forward along the short route only when the wind is favorable, so
recorded forward moves are almost always followed by reaching the goal,
and a contrastive critic trained on these continuations values them
highly. A learner that cannot observe the wind, however, would move
forward regardless of it: conditioning on a forward action selects
favorable wind, so the observational transition
$P(s'\mid s,\mathrm{forward})$ never leads into the lava, whereas the
interventional transition $P_{\mathrm{forward}}(s'\mid s)$ also covers
unfavorable wind.
\end{example}

\subsection{Counterfactual Modeling}
\label{sec:counterfactual-modeling}

Evaluating a policy requires the interventional transition
$P_a(s'\mid s)$, whereas the offline data provide only the
observational transition $P(s'\mid s,a)$. Under unobserved confounding,
the former is generally not identifiable from the latter
\citep{pearl2009causality,bareinboim2022pearlhierarchy}: different
confounded MDPs can induce the same observational distribution but
different interventional transitions. Consequently, the occupancy
$d^\pi(s^*\mid s,a)$ in Eq.~\eqref{eq:goal-occupancy} is generally not
identifiable either \citep{kallus2020confounding,pmlr-v286-zhang25d}.
Rather than treating $P_a(s'\mid s)$ as a single unknown, we use counterfactual
reasoning to separate its identified and unidentified components. For an action $a\in\mathcal A$, the counterfactual (potential outcome)
variable $S^{(t+1)}_{A^{(t)}=a}=f_S(S^{(t)},a,U^{(t)})$ denotes the next
state that would have been observed had the action at time $t$ been set
to $a$, with the state and the unobserved confounder held at their
realized values \citep{pearl2009causality}. Unlike interventional
quantities, counterfactual variables can be evaluated jointly with the
factual ones under the same realization of $U^{(t)}$, relating the
outcome of action $a$ to the action $a'$ actually selected by the
demonstrator. Conditioning on the factual state and
action yields the counterfactual transition probability
\begin{equation}
\label{eq:ctf-transition}
\mathcal T(s,a,a',s')
=
P(S^{(t+1)}_{A^{(t)}=a}{=}s'\mid S^{(t)}{=}s,A^{(t)}{=}a').
\end{equation}
It answers the counterfactual query: ``What would the next state have
been had the action at state $s$ been set to $a$, given that the
demonstrator in fact selected $a'$?'' This construction follows the
same counterfactual reasoning as the effect of treatment on the treated
(ETT) \citep{shpitser2009uai-effects}, which we extend here to state
transitions in MDPs. Conditioning on $A^{(t)}=a'$ retains information
about the confounder under which the demonstrator acted.

The counterfactual transition connects the observational and
interventional transitions. When $a=a'$, the consistency axiom
\citep{galles1998axiomatic,pearl2009causality} gives
$\mathcal T(s,a,a,s')=P(s'\mid s,a)$, which is identified from the
offline data. Averaging over the
demonstrator's action instead recovers the interventional transition by
the law of total probability,
\begin{equation}
\label{eq:ett-decomposition}
\sum_{a'\in\mathcal A}
\mathcal T(s,a,a',s')\,P(a'\mid s)
=
P_a(s'\mid s),
\end{equation}
where the sum becomes an integral for continuous actions. A counterfactual transition is therefore compatible with the offline
data as long as it coincides with the observational transition along
the direction $a=a'$; the remaining entries, with $a\neq a'$, describe
the outcomes of actions the demonstrator did not take and are free
parameters. Thus, the compatible set is
$\mathbb T=\{\mathcal T:\mathcal T(s,a,a,\cdot)=P(\cdot\mid s,a)$ for
all $s,a\}$, where each $\mathcal T(s,a,a',\cdot)$ is a distribution
over $\mathcal S$.

Building on the counterfactual transition, we define the counterfactual
occupancy measure
\begin{equation}
\label{eq:ctf-occupancy}
\begin{aligned}
Q^\pi(s,a,a')
&=
(1-\gamma)
\sum_{k=0}^{\infty}
\gamma^k\,
P\bigl(S^{(t+k)}_{A^{(t)}=a,\pi}=s^*
\\
&\qquad\qquad
\mid S^{(t)}=s,A^{(t)}=a'\bigr),
\end{aligned}
\end{equation}
where $S^{(t+k)}_{A^{(t)}=a,\pi}$ is the state that would have been
reached at time $t+k$ had $A^{(t)}$ been set to $a$ and subsequent
actions followed $\pi$, given that the demonstrator in fact selected
$a'$. Averaging over the demonstrator's
action recovers the Q-function,
$Q^\pi(s,a)=\sum_{a'}Q^\pi(s,a,a')\,P(a'\mid s)$, as in
Eq.~\eqref{eq:ett-decomposition}. Since the unobserved confounder is
drawn independently at each step, the counterfactual occupancy
satisfies the counterfactual Bellman equation
\begin{equation}
\label{eq:ctf-bellman}
\begin{aligned}
Q^\pi(s,a,a')
&=
(1-\gamma)\mathbbm{1}\{s=s^*\}
\\
&\quad+
\gamma\sum_{s'}\mathcal T(s,a,a',s')
\sum_{\tilde a,\tilde a'}
\pi(\tilde a\mid s')
\\
&\qquad\times
P(\tilde a'\mid s')\,
Q^\pi(s',\tilde a,\tilde a').
\end{aligned}
\end{equation}
That is, the next state follows the counterfactual transition, after
which the policy selects $\tilde a$ while the demonstrator's action
$\tilde a'$ is again drawn from $P(\cdot\mid s')$.

The counterfactual Bellman equation depends on the unobserved
confounder only through the counterfactual transition. Replacing the
true counterfactual transition with a candidate $\mathcal T$ yields a
candidate MDP $\mathcal M_{\mathcal T}$, and we use the subscript
$\mathcal T$ for quantities computed in $\mathcal M_{\mathcal T}$. In
particular, solving Eq.~\eqref{eq:ctf-bellman} with any compatible
$\mathcal T\in\mathbb T$ gives the counterfactual occupancy measure
$Q^\pi_{\mathcal T}(s,a,a')$, whose average
$Q^\pi_{\mathcal T}(s,a)=\sum_{a'}Q^\pi_{\mathcal T}(s,a,a')\,P(a'\mid s)$
is the Q-function of $\pi$ in $\mathcal M_{\mathcal T}$
(Lemma~\ref{lem:bellman}). By consistency, $\mathbb T$ contains the true counterfactual transition
of $\mathcal M$. Since the offline data cannot
distinguish it from the other compatible ones, minimizing over
$\mathbb T$ gives a worst-case occupancy measure, which lower-bounds
the true one.

\begin{proposition}
\label{prop:pessimistic-occupancy}
For every counterfactual transition $\mathcal T\in\mathbb T$,
the counterfactual Bellman equation~\eqref{eq:ctf-bellman} has a unique
solution $Q^\pi_{\mathcal T}(s,a,a')$, and the worst-case occupancy
measure for a policy $\pi(a\mid s)$,
\[
\underline Q^\pi(s,a)
=
\inf_{\mathcal T\in\mathbb T}
\sum_{a'\in\mathcal A}
Q^\pi_{\mathcal T}(s,a,a')\,P(a'\mid s)
\]
satisfies $\underline Q^\pi(s,a)\leq Q^\pi(s,a)$ for every $(s,a)$.
\end{proposition}

The proof is given in
Appendix~\ref{app:theory}. Since $\mathbb T$ and
$P(a'\mid s)$ are determined by the observational data,
$\underline Q^\pi$ lower-bounds the probability of reaching the goal
without access to the unobserved confounder.

\section{Causal Contrastive Reinforcement Learning}
\label{sec:method}

We now present Causal Contrastive Reinforcement Learning (CauCRL), which approximates the pessimistic construction of Sec.~\ref{sec:counterfactual-modeling} with learned models. CauCRL retains the observed dynamics for the factual component and models the unidentified component pessimistically. It learns a pessimistic counterfactual transition model (Sec.~\ref{sec:counterfactual-training}), generates counterfactual rollouts under learner-selected actions whose states serve as positive samples for the contrastive critic, and improves the policy with this critic (Sec.~\ref{sec:counterfactual-rollouts}). Algorithm~\ref{alg:caucrl} summarizes the procedure.

\subsection{Learning Counterfactual Transitions}
\label{sec:counterfactual-training}

Computing the worst-case occupancy measure in
Proposition~\ref{prop:pessimistic-occupancy} requires characterizing the
worst-case counterfactual transitions within $\mathbb T$. In discrete
domains, the entries with $a\neq a'$ are unrestricted, and the worst
case is given by the Manski bounds \citep{manski1990nonparametric}: the
interventional transition follows the observational one with
probability $P(a\mid s)$ and is otherwise unconstrained. Such bounds
have been applied to off-policy evaluation \citep{pmlr-v286-zhang25d}
and deep Q-learning \citep{NEURIPS2025_c9a177f0}. In continuous
domains, however, the demonstrator selects any particular action with
probability zero, so the constraint
$\mathcal T(s,a,a,\cdot)=P(\cdot\mid s,a)$ alone does not restrict
nearby actions, and consistent bound estimation requires additional
smoothness constraints \citep{tan2024consistency}.

We therefore assume that the counterfactual transition is Lipschitz
continuous in all of its conditioning arguments: for some constant
$L<\infty$,
\[
\begin{aligned}
&W\big(\mathcal T(s_1,a_1,a'_1,\cdot),\,\mathcal T(s_2,a_2,a'_2,\cdot)\big)
\\
&\;\leq
L\big(d_{\mathcal S}(s_1,s_2)+d_{\mathcal A}(a_1,a_2)
+d_{\mathcal A}(a'_1,a'_2)\big)
\end{aligned}
\]
for all $s_1,s_2\in\mathcal S$ and $a_1,a_2,a'_1,a'_2\in\mathcal A$, where
$W$ denotes the Wasserstein-1 distance between distributions over
$\mathcal S$. Along the action arguments, this assumption mirrors the
smoothness conditions of \citet{kim2018identification}. Continuity in the evaluated action
$a$ corresponds to smooth treatment response (STR): the next-state
distribution changes gradually with the intervened action. Continuity
in the demonstrator's action $a'$ corresponds to smooth treatment
selection (STS): situations in which the demonstrator selects nearby
actions lead to similar counterfactual outcomes. Continuity in the
current state $s$ encodes the locality of the transition dynamics typical of
RL. For example, in grid-world navigation, agents moving in neighboring cells are likely to lead to a collision.

As illustrated in Fig.~\ref{fig:lipschitz-transport}, consistency
identifies $\mathcal T$ along the diagonal $a=a'$. An unidentified
entry $(s,a,a')$ with $a\neq a'$ reaches the diagonal by moving $a$ and
$a'$ to a common value $x$ between them, a path of total length
$d_{\mathcal A}(a,a')$. Lipschitz continuity limits how far
probability mass can be transported along this path,
\[
W\big(\mathcal T(s,a,a',\cdot),\,P(\cdot\mid s,x)\big)
\leq
L\,d_{\mathcal A}(a,a').
\]
The counterfactual transition thus lies in a Wasserstein ball around
the identified transitions: the outcomes of actions close to the
demonstrator's remain close to the observed ones, while the ball widens
as the evaluated action moves away. The worst-case occupancy measure is
obtained by restricting the minimization in
Proposition~\ref{prop:pessimistic-occupancy} to Lipschitz-continuous
elements of $\mathbb T$.

\begin{figure}[t]
\centering
\begin{tikzpicture}[
    x={(1.0cm,0cm)}, y={(0.6cm,0.8cm)}, z={(0cm,0.85cm)},
    every node/.style={font=\small},
    dens/.style={draw=black, fill=gray!25, line width=0.5pt},
    cand/.style={draw=black, dashed, line width=0.6pt},
]
\fill[gray!6] (0,0,0) -- (3.2,0,0) -- (3.2,3.2,0) -- (0,3.2,0) -- cycle;
\draw[gray!45] (0,0,0) -- (3.2,0,0) -- (3.2,3.2,0) -- (0,3.2,0) -- cycle;
\draw[->] (0,0,0) -- (3.55,0,0) node[right] {$a$};
\draw[->] (0,0,0) -- (0,3.55,0) node[anchor=south east, inner sep=1pt] {$a'$};
\draw[->] (0,0,0) -- (0,0,2.4) node[above] {$s'$};
\draw[line width=1.6pt] (0,0,0) -- (3.2,3.2,0);
\node[anchor=west, inner sep=2pt] at (3.2,3.2,0) {$a=a'$};
\draw[line width=0.9pt] (2.6,0.6,0) -- (1.6,0.6,0);
\draw[->, line width=0.9pt] (1.6,0.6,0) -- (1.6,1.53,0);
\node[below=1pt] at (2.1,0.6,0) {$L\,d_{\mathcal A}(a,a')$};
\foreach \t/\m in {2.6/1.4, 1.6/1.15, 0.6/0.95}{
  \filldraw[dens] (\t,\t,0)
    -- plot[domain=0:2.1, samples=60, smooth]
       ({\t+0.45*exp(-(\x-\m)*(\x-\m)/0.08)},{\t},{\x})
    -- (\t,\t,2.1) -- cycle;
  \fill (\t,\t,0) circle (1.6pt);
}
\node[anchor=south] at (1.75,1.6,2.12) {$P(\cdot\mid s,x)$};
\foreach \m in {0.65, 0.95, 1.25}{
  \draw[cand] plot[domain=0:2.1, samples=60, smooth]
    ({2.6+0.45*exp(-(\x-\m)*(\x-\m)/0.08)},{0.6},{\x});
}
\draw[gray!60] (2.6,0.6,0) -- (2.6,0.6,2.1);
\fill (2.6,0.6,0) circle (1.6pt);
\node[anchor=west] at (3.1,0.6,1.0) {$\mathcal T(s,a,a',\cdot)$};
\end{tikzpicture}
\caption{
Lipschitz continuity of the counterfactual transition. Along the
diagonal $a=a'$ (thick line), $\mathcal T(s,a,a,\cdot)=P(\cdot\mid s,a)$
is identified (filled densities over $s'$). An off-diagonal entry
$\mathcal T(s,a,a',\cdot)$ is not identified (dashed densities). The
shortest transport to the diagonal is L-shaped (arrow): it moves the
evaluated action from $a$ and the demonstrator's action from $a'$ to a
common value $x$, with total length $d_{\mathcal A}(a,a')$.
}
\label{fig:lipschitz-transport}
\end{figure}
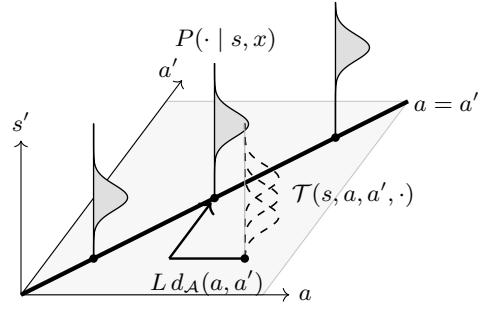

We model a pessimistic counterfactual transition with a conditional
generative model $\mathcal T_\eta$: a neural network $g_\eta$ maps the state $s$,
the evaluated action $a$, the demonstrator's action $a'$, and noise
$\varepsilon\sim\mathcal N(0,I)$ to a next state
$\tilde s'=g_\eta(s,a,a',\varepsilon)$. Spectral normalization of its
weights \citep{miyato2018spectral} bounds the Lipschitz constant of
$g_\eta$ in $(s,a,a')$ for every $\varepsilon$; with shared noise,
this bound carries over to the Wasserstein distance between next-state
distributions.

The model is trained with two objectives. The observational objective
fits the entries with $a=a'$ to the positive demonstrations
$\mathcal D^{+}$ drawn from $P^+$,
\[
\mathcal L^{+}(\eta)
=
\mathbb E_{(s,a,s')\sim\mathcal D^{+}}
\left[
\ell\!\left(
\mathcal T_\eta(s,a,a,\cdot),
s'
\right)
\right],
\]
where $\ell$ is a proper scoring rule, such as the negative
log-likelihood or the energy score. Minimizing $\mathcal L^{+}$
enforces the consistency constraint that defines $\mathbb T$; the
remaining entries are chosen pessimistically. Since the negative
demonstrations record only states, the interventional objective pairs
each state--action pair $(s,a')$ from $\mathcal D^{+}$ with an
alternative action $a\neq a'$ drawn from a proposal $\rho(\cdot\mid s)$,
such as a uniform distribution, and with an independent negative state
$s^-$ from $\mathcal D^{-}$:
\[
\mathcal L^{-}(\eta)
=
\mathbb E_{(s,a')\sim\mathcal D^{+},\,a\sim\rho,\,s^-\sim\mathcal D^{-}}
\left[
\ell\!\left(
\mathcal T_\eta(s,a,a',\cdot),
s^-
\right)
\right].
\]
Minimizing $\mathcal L^{-}$ steers the entries with $a\neq a'$ toward
states reached by non-expert behavior; without further constraints, its
minimizer is $\mathcal T(s,a,a',\cdot)=P^-$. When $P^-$ concentrates on
absorbing failure states such as the lava, this coincides with the
worst case of the Manski bounds, which places all unidentified mass on
states from which the goal cannot be reached, and thus approximates the
minimizer in Proposition~\ref{prop:pessimistic-occupancy}. The
combined objective is given by
\begin{equation}
\mathcal L_{\mathcal T}(\eta)
=
\mathcal L^{+}(\eta)
+
\mathcal L^{-}(\eta)
\label{eq:ett-loss}
\end{equation}
is minimized over spectrally normalized networks: $\mathcal L^{+}$
anchors the model to the identified transitions, $\mathcal L^{-}$
selects pessimistic outcomes among the unidentified ones, and the
Lipschitz constraint keeps $\mathcal T_\eta$ within the Wasserstein
ball of radius $L\,d_{\mathcal A}(a,a')$. This prevents arbitrarily
adverse predictions and lets the model generalize to action pairs not
covered by the data. Implementation details are given in
Appendix~\ref{app:implementation}.

\subsection{Counterfactual Rollouts}
\label{sec:counterfactual-rollouts}
\label{sec:rollout-construction}

Evaluating a policy $\pi$ requires next states drawn from the
interventional transition $P_a(s'\mid s)$, which by
Eq.~\eqref{eq:ett-decomposition} averages the counterfactual transition
over the demonstrator's action. Since $P(A\mid S)$ is identified from
the demonstrations, we estimate it from $\mathcal D^{+}$ with a
statistical model $\mu_\phi(a'\mid s)$, the nominal behavioral policy. Any
conditional density estimator can be used; our choices are detailed in
Appendix~\ref{app:implementation}. For a counterfactual transition $\mathcal T$,
averaging over the demonstrator's action defines the induced transition
\begin{equation}
\label{eq:induced-transition}
P^{\mathcal T}_a(s'\mid s)
=
\sum_{a'\in\mathcal A}
\mathcal T(s,a,a',s')\,P(a'\mid s),
\end{equation}
which coincides with $P_a(s'\mid s)$ for the true counterfactual
transition. As previously defined, we write $\mathcal M_{\mathcal T}$ for the MDP with
transition $P^{\mathcal T}_a$ (Sec.~\ref{sec:counterfactual-modeling}).

\begin{algorithm}[t]
\caption{Causal Contrastive RL (CauCRL)}
\label{alg:caucrl}
\begin{algorithmic}[1]
\Require $\mathcal D^{+}$, $\mathcal D^{-}$, rounds $R$, horizon $H$
\State Fit nominal policy $\mu_\phi$ on $\mathcal D^{+}$
\State Fit $\mathcal T_\eta$ by minimizing Eq.~\eqref{eq:ett-loss}
\State Initialize critic $f_\omega$ and actor $\pi_\theta$
\For{$r=1,\ldots,R$}
    \State $\widetilde{\mathcal D}\gets\Call{CtfRollout}{\pi_\theta,\mathcal T_\eta,\mu_\phi}$
    \State $\pi_\theta\gets\Call{PolicyImprovement}{\widetilde{\mathcal D},\pi_\theta}$
\EndFor
\State \Return $\pi_\theta$
\Procedure{CtfRollout}{$\pi,\mathcal T,\mu$}
    \State Sample $(s,a)\sim\mathcal D^{+}$; set
    $\tilde s^{(t)}=s$, $\tilde a^{(t)}=a$
    \For{$h=t,\ldots,t+H-1$}
        \State Sample
        $a^{\prime(h)}\sim\mu(\cdot\mid\tilde s^{(h)})$
        \State Sample
        $\tilde s^{(h+1)}\sim
        \mathcal T
        (\tilde s^{(h)},\tilde a^{(h)},a^{\prime(h)},\cdot)$
        \State Sample
        $\tilde a^{(h+1)}\sim
        \pi(\cdot\mid\tilde s^{(h+1)})$
    \EndFor
    \State \Return rollouts $\widetilde{\mathcal D}$
\EndProcedure
\Procedure{PolicyImprovement}{$\widetilde{\mathcal D},\pi$}
    \For{each gradient step}
        \State Sample $(s,a,\tilde s^+)$ from $\widetilde{\mathcal D}$ and
        $s^-\sim\mathcal D^{-}$
        \State Update $f_\omega$ by Eq.~\eqref{eq:caucrl-nce} and $\pi$ by
        Eq.~\eqref{eq:caucrl-actor}
    \EndFor
    \State \Return $\pi$
\EndProcedure
\end{algorithmic}
\end{algorithm}

A counterfactual rollout samples from
$P^{\mathcal T_\eta}_a$ under the policy $\pi$, with $\mu_\phi$ in place
of $P(a'\mid s)$ (\textsc{CtfRollout} in Algorithm~\ref{alg:caucrl}). It starts from a
state--action pair $(s,a)$ at time $t$, with $\tilde s^{(t)}=s$ and
$\tilde a^{(t)}=a$; at each step $h\geq t$, the demonstrator's action
$a^{\prime(h)}$ is drawn from the nominal policy, the next state
$\tilde s^{(h+1)}$ from $\mathcal T_\eta$, and the learner's next action
from $\pi$; $a'$ is resampled even at the first step, since the
recorded action would give the confounded
$\mathcal T(s,a,a,\cdot)=P(\cdot\mid s,a)$.
Algorithm~\ref{alg:caucrl} uses the current actor as the rollout policy,
which may also be kept fixed across rounds.

Replacing the positive samples from $P^+$ with these rollouts
yields the counterfactual NCE objective
\begin{align}
\max_{\omega}\;
\mathbb{E}_{\substack{(s,a,\tilde s^+)\sim\textsc{CtfRollout}\\ s^-\sim P^-}}
&\Big[
\log \sigma\!\left(f_\omega(s,a,\tilde s^+)\right)
\label{eq:caucrl-nce}
\\[-8pt]
&+
\log\!\left(
1-\sigma\!\left(f_\omega(s,a,s^-)\right)
\right)
\Big],
\nonumber
\end{align}
where each pair $(s,a)$ from $\mathcal D^{+}$ takes the rollout state
$\tilde s^+=\tilde S^{(t+K)}$ as its positive goal, with
$P(K=k)=(1-\gamma)\gamma^k$ for $k\geq0$, and the negative goals are
unchanged. The following proposition shows that the optimal critic
recovers the Q-function of $\pi$ in the MDP induced by the
counterfactual transition.\footnote{The constant $1/P^-(s^*)$ requires
$P^-(s^*)>0$. This can always be ensured by mixing random states from
the whole dataset into the negative samples, as in the original CRL
\citep{NEURIPS2022_e7663e97}. These additional states serve only as
negatives for the critic in Eq.~\eqref{eq:caucrl-nce}; they are not
used to train the counterfactual transition in
Eq.~\eqref{eq:ett-loss}.}

\begin{proposition}
\label{prop:ctf-critic}
Suppose that \textsc{CtfRollout} in Eq.~\eqref{eq:caucrl-nce} uses a
counterfactual transition $\mathcal T$ and the nominal policy
$\mu=P(\cdot\mid s)$. The critic that optimizes
Eq.~\eqref{eq:caucrl-nce} is the Q-function of $\pi$ in
$\mathcal M_{\mathcal T}$, up to a multiplicative constant
$1/P^-(s^*)$:
\[
\exp f^*(s,a,s^*)
=
\frac{1}{P^-(s^*)}\cdot Q^\pi_{\mathcal T}(s,a).
\]
\end{proposition}

The proof is given in Appendix~\ref{app:theory}. Since
$P^-(s^*)$ does not depend on the action,
$\arg\max_a f^*(s,a,s^*)=\arg\max_a Q^\pi_{\mathcal T}(s,a)$:
the critic retains the value interpretation of contrastive RL, now in
the MDP induced by the counterfactual transition. In practice,
rollouts are truncated at a horizon $H$; sampling $K$ only up to $H$
changes the distribution of positive goals by at most $\gamma^{H+1}$ in
total variation. Details of the rollout construction are given in
Appendix~\ref{app:implementation}.

\textbf{Policy improvement.}
CauCRL retains the actor update used in standard offline CRL. The actor
selects actions with high critic values while remaining close to the
actions in the offline demonstrations:
\begin{equation}
\label{eq:caucrl-actor}
\! \max_{\theta}\;
\mathbb E_{\substack{(s,a')\sim\mathcal D^{+}\\ a\sim\pi_\theta(\cdot\mid s)}}
\Big[
    f_\omega(s,a,s^*)
    +
    \lambda_{\mathrm{BC}}
    \log
    \pi_\theta(a'\mid s)
\Big].
\end{equation}
The critic evaluates the policy's own action $a$, while the
behavioral-cloning term, weighted by $\lambda_{\mathrm{BC}}$, keeps
$\pi_\theta$ close to the recorded action $a'$. By Proposition~\ref{prop:ctf-critic}, higher critic values
mean higher Q-values in $\mathcal M_{\mathcal T}$, so this update
improves the policy in that MDP.

\begin{proposition}
\label{prop:policy-improvement}
Assume that states and actions are tabular, $\mu=P(\cdot\mid s)$, and
$P^-(s^*)>0$. Let the critic $f$ be $\delta$-accurate for rollouts of
$\pi$ under $\mathcal T$, i.e.,
$|P^-(s^*)\exp f(s,a,s^*)-Q^\pi_{\mathcal T}(s,a)|\leq\delta$ for all
$(s,a)$, and let $\pi'(s)\in\arg\max_a f(s,a,s^*)$ be the policy
obtained after one iteration of \textsc{PolicyImprovement} with
$\lambda_{\mathrm{BC}}=0$. Then $\pi'$ achieves a higher goal-reaching
occupancy than $\pi$ in $\mathcal M_{\mathcal T}$, up to the critic
error:
\[
Q^{\pi'}_{\mathcal T}(s,a)
\geq
Q^{\pi}_{\mathcal T}(s,a)-\frac{2\gamma\delta}{1-\gamma}
\quad\text{for all }(s,a).
\]
\end{proposition}

The proof (Appendix~\ref{app:theory}) mirrors the
approximate policy improvement result of
\citet[Lemma~4.2]{NEURIPS2022_e7663e97}, with $\mathcal M_{\mathcal T}$
in place of the true environment. For the Bayes-optimal
critic of Proposition~\ref{prop:ctf-critic}, $\delta=0$ and $\pi'$
improves upon $\pi$ exactly, so alternating \textsc{CtfRollout} and
\textsc{PolicyImprovement} in Algorithm~\ref{alg:caucrl} performs
approximate policy iteration in $\mathcal M_{\mathcal T_\eta}$.

Proposition~\ref{prop:policy-improvement} guarantees improvement in
$\mathcal M_{\mathcal T}$. When $\mathcal T$ is the worst-case
counterfactual transition of Proposition~\ref{prop:pessimistic-occupancy},
$Q^\pi_{\mathcal T}$ lower-bounds the true occupancy. With the
pessimistic $\mathcal T_\eta$ learned in
Sec.~\ref{sec:counterfactual-training}, CauCRL improves an approximation
of this lower bound as a proxy for policy improvement in the real
environment.

\section{Experiments}
\label{sec:experiments}

\begin{figure}[t]
    \centering
    \includegraphics[
        height=0.29\columnwidth,
        trim=6 6 6 6,
        clip
    ]{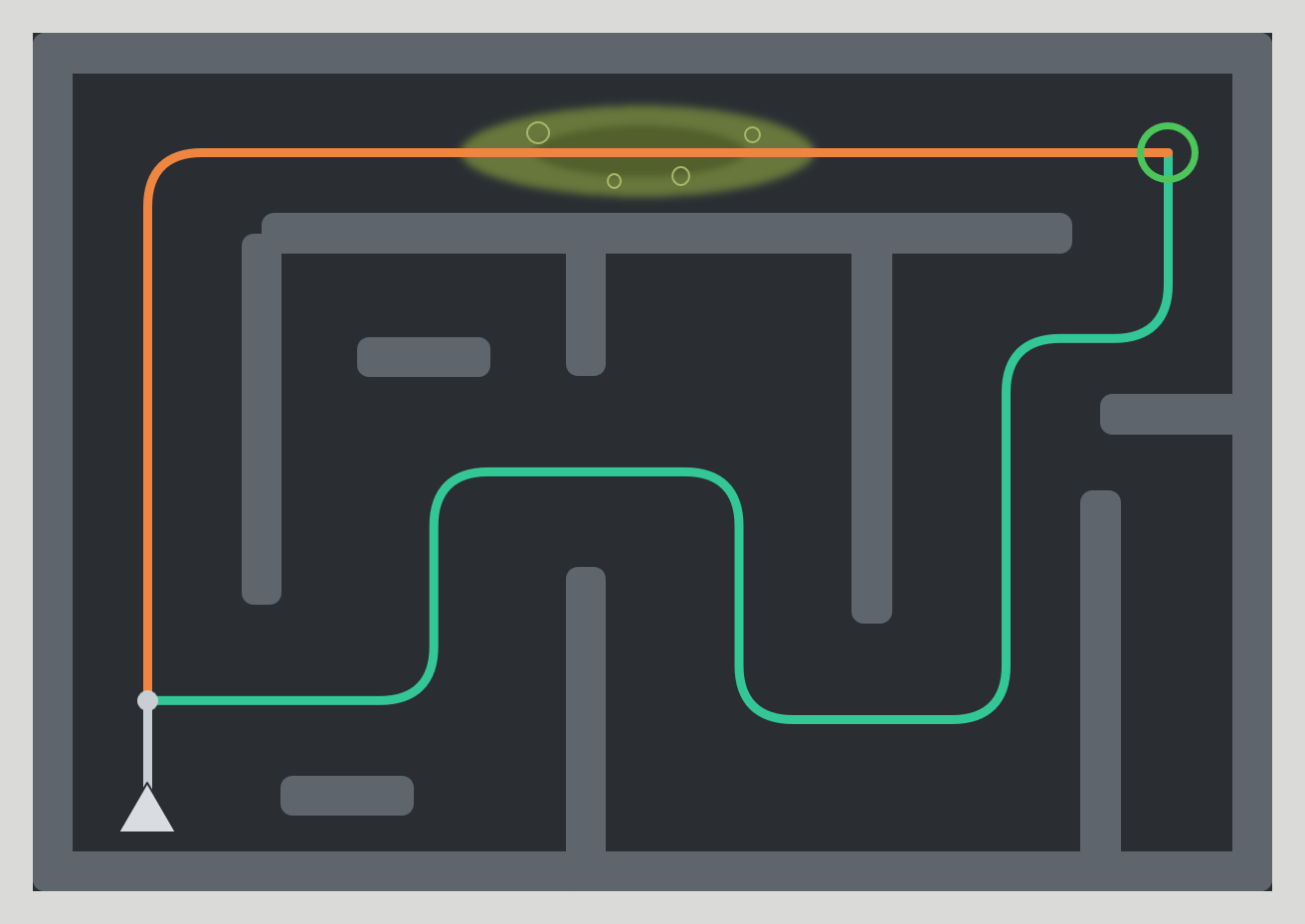}
    \hspace{0.025\columnwidth}
    \includegraphics[
        height=0.29\columnwidth
    ]{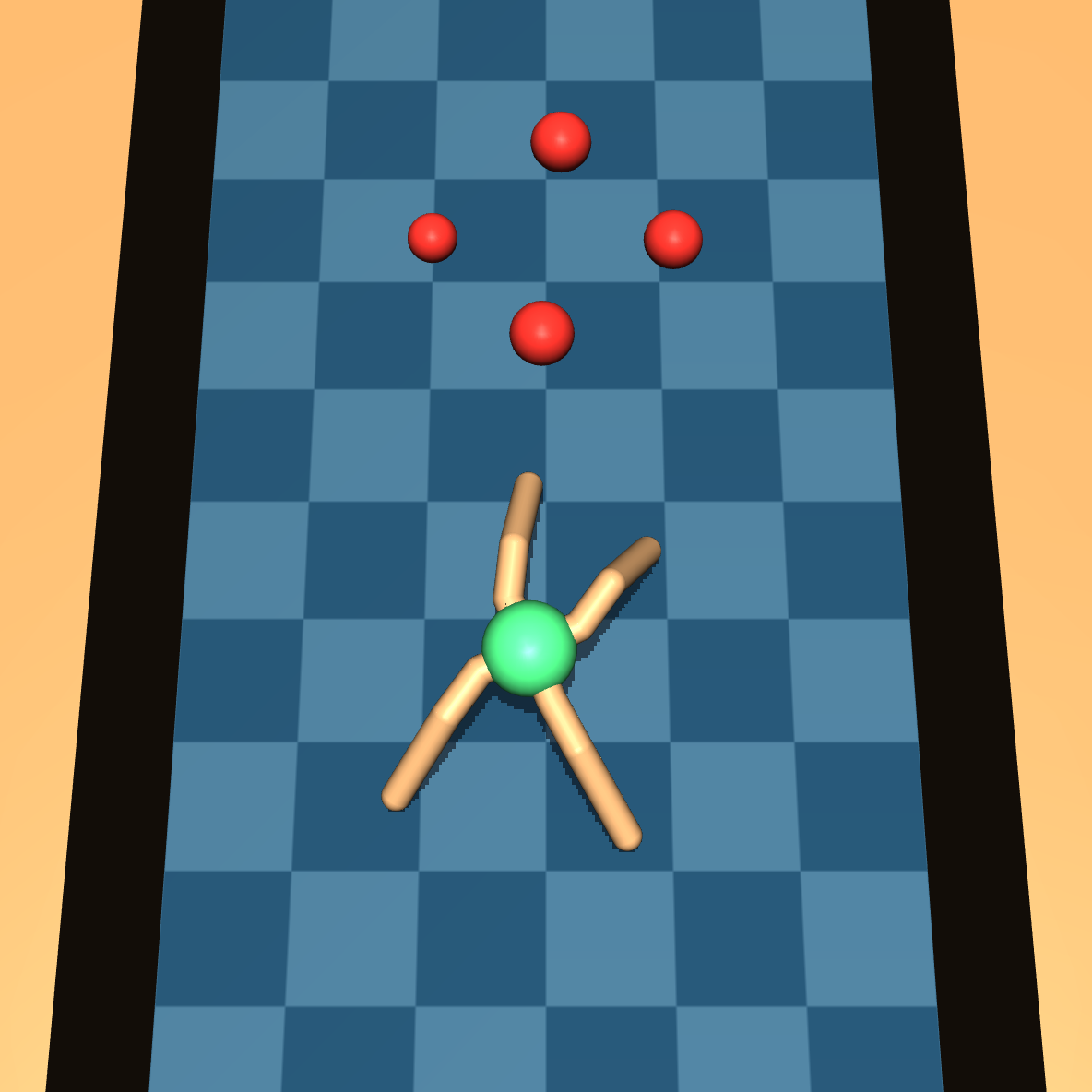}
\caption{
Continuous-control environments: (a) PointMaze with a stochastic swamp
region (yellow) and (b) AntMaze with falling-rock hazards (red).
}
    \label{fig:environments}
\end{figure}

We evaluate CauCRL on three goal-conditioned navigation tasks with
unobserved confounding and increasing control complexity: Windy Corridor
with discrete actions, PointMaze with continuous planar control, and
AntMaze with eight-dimensional continuous actions
(Fig.~\ref{fig:environments}).
In each environment, the demonstrator observes latent information that
affects both its actions and subsequent transitions but is hidden from
the learner and all evaluated methods.

We compare CauCRL against four goal-conditioned offline learning
baselines.
\textbf{CRL} is standard contrastive reinforcement
learning~\citep{NEURIPS2022_e7663e97} and our primary baseline; it uses
the same demonstrations and actor regularization as CauCRL but trains
its critic on recorded rather than counterfactual future goals, so
comparing the two isolates the effect of counterfactual rollouts.
\textbf{GCBC} performs behavioral cloning on hindsight-relabeled
trajectories~\citep{lynch2020learning,ghosh2021learning}.
\textbf{GCIQL} is a goal-conditioned variant of implicit
Q-learning~\citep{kostrikov2022offline}, and
\textbf{HIQL} uses latent future states as high-level actions for
long-horizon goal-conditioned control~\citep{park2023hiql}.
All baselines treat recorded next states as outcomes of the recorded
actions and do not model the confounder.
For continuous-control tasks, CRL and CauCRL share the actor--critic
architecture and optimizer; full setup and training details are in
Appendices~\ref{app:pointmaze-setup}--\ref{app:hyperparameters}.

We report mean goal-reaching success and, where applicable, robust-route
rates over five seeds, with $\pm$ s.e.m.\ over the final 10\%
of evaluation checkpoints; per-seed results are in
Appendix~\ref{sec:result-details}.

\begin{figure*}[t]
  \centering
  \includegraphics[width=\textwidth]{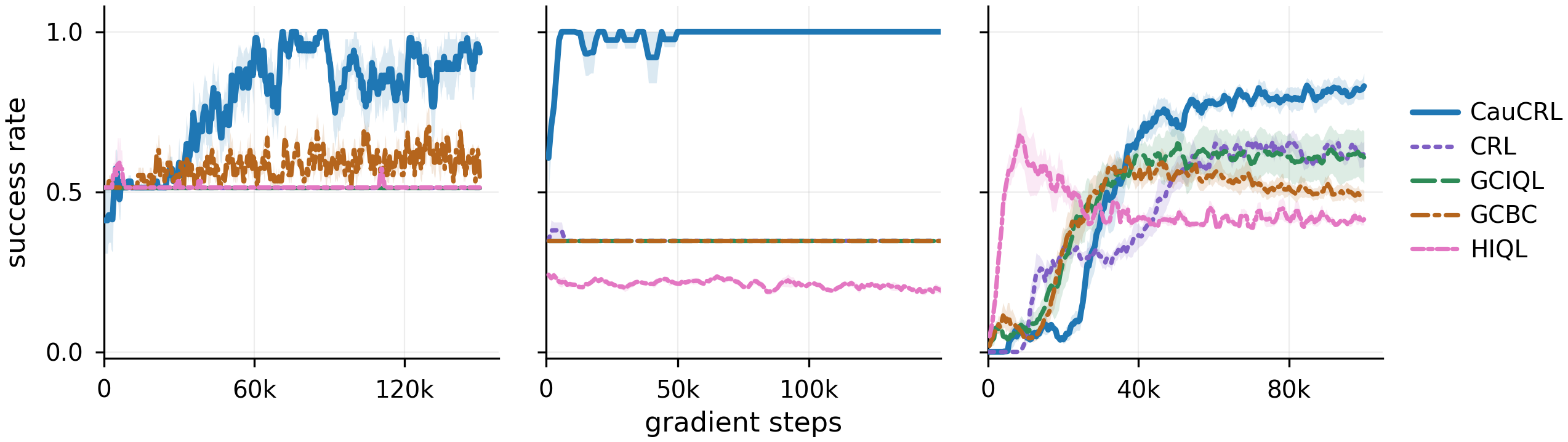}
\caption{
Success rates over training on Windy Corridor (left), PointMaze (middle),
and AntMaze (right). Curves show means over five random seeds, with
shaded regions indicating standard errors.
}
  \label{fig:results}
\end{figure*}

\textbf{Experiment 1: Windy Corridor.}
We revisit the $15\times15$ Windy Corridor task of
Example~\ref{ex:windy} (Appendix~\ref{app:windy-setup}), with a short
passage adjacent to lava and a longer safe route.
The demonstrator primarily takes the shortcut and uses the observed wind
to pause when forward motion is unsafe, while a smaller fraction of
demonstrations follows the safe route.
Because the wind is resampled at every step and hidden from the learner,
only the safe route reaches the goal reliably.

In Fig.~\ref{fig:results} (left), CauCRL achieves a mean success
rate of $0.919\pm0.036$ and a robust-route rate of $1.00$, despite
the safe route requiring about 39 steps instead of 13.
Although its success varies across seeds ($0.797$ to $1.000$), every
evaluation episode takes the safe route.
All baselines remain below $0.60$ success; CRL, GCIQL, and HIQL reach
$0.513$ on every seed and converge to the confounded shortcut.
Because demonstrations pair shortcut actions almost exclusively with
successful outcomes, confounding-unaware critics overvalue this route.
CauCRL assigns pessimistic outcomes to moving forward where the
demonstrator might have waited, so its critic prefers the safe route.

\textbf{Experiment 2: PointMaze.}
The agent observes a four-frame history of planar positions and the goal
(16 dimensions) and selects two-dimensional velocities
(Appendix~\ref{app:pointmaze-setup}).
The shortcut crosses three swamp cells, each active with probability
$0.30$ at every step; entering an active cell freezes the agent for the
remainder of the 50-step episode.
The demonstrator observes the swamp condition and primarily follows the
shortcut, waiting when necessary, while a smaller set of demonstrations
takes the detour.

In Fig.~\ref{fig:results} (middle), CauCRL achieves both a mean success
rate and detour rate of $1.00$.
CRL, GCBC, and GCIQL each achieve $0.347$ success and never take the
detour, close to the probability $(1-0.30)^3\approx0.343$ that all three
cells are inactive when crossed.
HIQL performs worse ($0.176\pm0.013$).
Since the swamp is hidden and resampled at every step, waiting cannot
improve these odds for the learner; CauCRL instead identifies the detour
as the reliable route.

\textbf{Experiment 3: AntMaze.}
The Ant has a 31-dimensional observation, an eight-dimensional action
space, and an 800-step horizon (Appendix~\ref{app:antmaze-setup}).
The direct route crosses two rockfall zones, each activated with
probability $0.40$ per episode, while a longer route avoids them;
contact with a falling rock terminates the episode.
The demonstrator observes the rockfall state, but the learner observes
only the Ant state and goal.
Because the demonstrator crosses only when the route is safe, its
trajectories contain no rock-contact failures and make direct-route
actions appear safe.
CauCRL additionally trains its counterfactual transition on forced
crossings that override a demonstrator decision to wait and expose the
resulting rock collisions.

In Fig.~\ref{fig:results} (right), CauCRL achieves
$0.815\pm0.030$ success, compared with $0.629\pm0.015$ for CRL, an
absolute improvement of $0.186$.
GCIQL reaches $0.616\pm0.076$, while GCBC and HIQL achieve
$0.498\pm0.022$ and $0.409\pm0.006$, respectively.
CauCRL's lowest seed ($0.712$) exceeds CRL's highest ($0.677$).

\section{Conclusion}

We study offline goal-conditioned reinforcement learning under unobserved confounding, where demonstrators may use information unavailable to the learner. Standard contrastive RL can then learn spurious action--outcome associations from observational future states. To address this problem, we introduce Causal Contrastive Reinforcement Learning (CauCRL), which learns a pessimistic counterfactual transition from positive and negative demonstrations and uses counterfactual rollouts for contrastive critic learning. Our causal formulation separates identified factual transitions from unidentified counterfactual outcomes and yields a worst-case goal-reaching occupancy that lower-bounds the true occupancy. Experiments across discrete and continuous-control tasks show that CauCRL consistently improves robustness over confounding-unaware baselines. These results highlight the value of counterfactual future-state supervision under confounding. Future work will evaluate CauCRL beyond maze-style navigation, including goal-conditioned manipulation and more diverse sources of unobserved confounding.

\subsection*{AI Use Statement}

In this work, we used generative AI tools to assist with
research ideation, methodological discussions, experimental
design and diagnostics, code development and debugging,
and checking mathematical arguments. We also used these
tools for retrieving and summarizing related work and for
drafting and revising portions of the manuscript.

All AI-assisted outputs were reviewed and verified by the
authors. Mathematical arguments were checked against the
stated assumptions and derivations, code was inspected
and validated through experiments, and literature summaries
were verified against the original sources. The authors
made all final decisions regarding the methodology,
experiments, interpretations, and presentation, and take
full responsibility for the content of this work.

\bibliographystyle{apalike}
\bibliography{references}
\appendix
\section{Proofs}
\label{app:theory}

This appendix collects the proofs of the three propositions in the main
text. Proposition~\ref{prop:pessimistic-occupancy} shows that minimizing
the counterfactual occupancy over all compatible counterfactual
transitions lower-bounds the true goal-reaching occupancy, which
underlies the pessimistic construction of
Sec.~\ref{sec:counterfactual-modeling};
Proposition~\ref{prop:ctf-critic} identifies the Bayes-optimal critic
trained on counterfactual rollouts, which underlies the counterfactual
NCE objective of Sec.~\ref{sec:counterfactual-rollouts}; and
Proposition~\ref{prop:policy-improvement} shows that greedy updates with
an approximately optimal critic improve the policy in the MDP induced by
the counterfactual transition. All three proofs follow the standard
dynamic-programming playbook: each quantity of interest is characterized
as the unique fixed point of a $\gamma$-contractive Bellman operator,
and identities and inequalities between these quantities are obtained
by verifying them for the corresponding operators.

Throughout, we fix a goal $s^*\in\mathcal S$ and a discount factor
$\gamma\in(0,1)$, and all functions are bounded and equipped with the
supremum norm $\|\cdot\|_\infty$. Sums over states and actions become
integrals for continuous spaces, and all arguments carry over verbatim.
Entries $(s,a,a')$ with $P(a'\mid s)=0$ receive zero weight in
Eq.~\eqref{eq:ctf-bellman} and in the definition of $\underline Q^\pi$,
so they play no role and can be fixed arbitrarily. For a policy $\pi$
and a counterfactual transition $\mathcal T\in\mathbb T$, we define the
counterfactual Bellman operator, acting on functions $q(s,a,a')$, as the
right-hand side of Eq.~\eqref{eq:ctf-bellman},
\[
\begin{aligned}
(\mathcal K^\pi_{\mathcal T}q)(s,a,a')
&=
(1-\gamma)\mathbbm{1}\{s=s^*\}
\\
&\quad+
\gamma\sum_{s'}\mathcal T(s,a,a',s')
\sum_{\tilde a,\tilde a'}
\pi(\tilde a\mid s')
\\
&\qquad\times
P(\tilde a'\mid s')\,q(s',\tilde a,\tilde a'),
\end{aligned}
\]
and the Bellman operator of $\pi$ in $\mathcal M_{\mathcal T}$, acting on
functions $q(s,a)$, as
\[
\begin{aligned}
(\mathcal B^\pi_{\mathcal T}q)(s,a)
&=
(1-\gamma)\mathbbm{1}\{s=s^*\}
\\
&\quad+
\gamma\sum_{s'}P^{\mathcal T}_a(s'\mid s)
\sum_{\tilde a}\pi(\tilde a\mid s')\,q(s',\tilde a),
\end{aligned}
\]
where $P^{\mathcal T}_a$ is the induced transition in
Eq.~\eqref{eq:induced-transition}. Finally,
$d^\pi_{\mathcal T}(s^*\mid s,a)$ denotes the goal-reaching occupancy of
$\pi$ in $\mathcal M_{\mathcal T}$, i.e., Eq.~\eqref{eq:goal-occupancy}
with $P_a$ replaced by $P^{\mathcal T}_a$, and $P^{\mathcal T}_{a,\pi}$
denotes the trajectory distribution in $\mathcal M_{\mathcal T}$ that
sets $A^{(t)}$ to $a$ and lets subsequent actions follow $\pi$. The
operator facts shared by the three proofs are encapsulated in the
following lemma.

\begin{lemma}
\label{lem:bellman}
Let $\pi$ be a policy and $\mathcal T\in\mathbb T$, and let $B$ denote
either $\mathcal K^\pi_{\mathcal T}$ or $\mathcal B^\pi_{\mathcal T}$.
\begin{itemize}
\item[(a)] $B$ is monotone, satisfies $B(q-c)=Bq-\gamma c$ for every
constant $c$, and is a $\gamma$-contraction in $\|\cdot\|_\infty$.
Consequently, $B$ has a unique bounded fixed point, which is the limit
of $B^nq$ as $n\to\infty$ for every bounded $q$.
\item[(b)] The fixed point of $\mathcal B^\pi_{\mathcal T}$ is
$d^\pi_{\mathcal T}(s^*\mid s,a)$.
\item[(c)] If $Q^\pi_{\mathcal T}(s,a,a')$ is the fixed point of
$\mathcal K^\pi_{\mathcal T}$, then
$Q^\pi_{\mathcal T}(s,a)=\sum_{a'}Q^\pi_{\mathcal T}(s,a,a')\,P(a'\mid s)$
is the fixed point of $\mathcal B^\pi_{\mathcal T}$. In particular,
$Q^\pi_{\mathcal T}(s,a)=d^\pi_{\mathcal T}(s^*\mid s,a)$, i.e.,
$Q^\pi_{\mathcal T}$ is the Q-function of $\pi$ in
$\mathcal M_{\mathcal T}$.
\end{itemize}
\end{lemma}

\begin{proof}
We first prove (a). Both operators are of the form
\[
(Bq)(x)
=
(1-\gamma)\mathbbm{1}\{s=s^*\}
+
\gamma\sum_{y}K(x,y)\,q(y),
\]
where $s$ is the state component of $x$. For
$\mathcal K^\pi_{\mathcal T}$, we have $x=(s,a,a')$,
$y=(s',\tilde a,\tilde a')$, and
\[
K(x,y)
=
\mathcal T(s,a,a',s')\,\pi(\tilde a\mid s')\,P(\tilde a'\mid s'),
\]
while for $\mathcal B^\pi_{\mathcal T}$, we have $x=(s,a)$,
$y=(s',\tilde a)$, and
\[
K(x,y)
=
P^{\mathcal T}_a(s'\mid s)\,\pi(\tilde a\mid s').
\]
In both cases, $K(x,\cdot)$ is a probability distribution, since
$\mathcal T(s,a,a',\cdot)$, $\pi(\cdot\mid s')$, $P(\cdot\mid s')$, and
$P^{\mathcal T}_a(\cdot\mid s)$ are. Since $K\geq0$, $B$ is monotone, and
since $\sum_yK(x,y)=1$, $B(q-c)=Bq-\gamma c$. Moreover,
\[
\begin{aligned}
|(Bq_1)(x)-(Bq_2)(x)|
&\leq
\gamma\sum_yK(x,y)\,|q_1(y)-q_2(y)|
\\
&\leq
\gamma\|q_1-q_2\|_\infty,
\end{aligned}
\]
so $B$ is a $\gamma$-contraction. The Banach fixed-point theorem gives
the remaining claims of (a).

We next prove (b). The $k=0$ term of Eq.~\eqref{eq:goal-occupancy} in
$\mathcal M_{\mathcal T}$ equals $(1-\gamma)\mathbbm{1}\{s=s^*\}$, since
$S^{(t)}=s$. For $k\geq1$, conditioning on the next state
$S^{(t+1)}=s'$ and the next action $A^{(t+1)}=\tilde a$, the Markov
property of $\mathcal M_{\mathcal T}$ gives
\[
\begin{aligned}
&P^{\mathcal T}_{a,\pi}\big(S^{(t+k)}=s^*\mid S^{(t)}=s\big)
\\
&\;=
\sum_{s'}P^{\mathcal T}_a(s'\mid s)
\sum_{\tilde a}\pi(\tilde a\mid s')
\\
&\qquad\times
P^{\mathcal T}_{\tilde a,\pi}\big(S^{(t+k)}=s^*\mid S^{(t+1)}=s'\big).
\end{aligned}
\]
Multiplying by $(1-\gamma)\gamma^k$, summing over $k\geq1$, and using
time-homogeneity to shift the time index by one, we obtain
\[
\begin{aligned}
d^\pi_{\mathcal T}(s^*\mid s,a)
&=
(1-\gamma)\mathbbm{1}\{s=s^*\}
\\
&\quad+
\gamma\sum_{s'}P^{\mathcal T}_a(s'\mid s)
\sum_{\tilde a}\pi(\tilde a\mid s')
\\
&\qquad\times
d^\pi_{\mathcal T}(s^*\mid s',\tilde a).
\end{aligned}
\]
Hence $d^\pi_{\mathcal T}$ is a fixed point of
$\mathcal B^\pi_{\mathcal T}$, and by (a) it is the unique one.

It remains to prove (c). Multiplying the fixed-point equation
$Q^\pi_{\mathcal T}=\mathcal K^\pi_{\mathcal T}Q^\pi_{\mathcal T}$ by
$P(a'\mid s)$ and summing over $a'$, the first term is unchanged, since
$\sum_{a'}P(a'\mid s)=1$. In the second term, Eq.~\eqref{eq:induced-transition}
gives
\[
\sum_{a'}\mathcal T(s,a,a',s')\,P(a'\mid s)
=
P^{\mathcal T}_a(s'\mid s),
\]
while the inner sum over $\tilde a'$ gives
\[
\sum_{\tilde a'}P(\tilde a'\mid s')\,
Q^\pi_{\mathcal T}(s',\tilde a,\tilde a')
=
Q^\pi_{\mathcal T}(s',\tilde a).
\]
Therefore
\[
\begin{aligned}
Q^\pi_{\mathcal T}(s,a)
&=
(1-\gamma)\mathbbm{1}\{s=s^*\}
\\
&\quad+
\gamma\sum_{s'}P^{\mathcal T}_a(s'\mid s)
\sum_{\tilde a}\pi(\tilde a\mid s')\,
Q^\pi_{\mathcal T}(s',\tilde a),
\end{aligned}
\]
that is, $Q^\pi_{\mathcal T}=\mathcal B^\pi_{\mathcal T}Q^\pi_{\mathcal T}$.
By (a) and (b),
$Q^\pi_{\mathcal T}(s,a)=d^\pi_{\mathcal T}(s^*\mid s,a)$. This
completes the proof.
\end{proof}

\begin{proof}[Proof of Proposition~\ref{prop:pessimistic-occupancy}]
The proof proceeds in three steps. We first note that the
counterfactual Bellman equation is well posed for every compatible
counterfactual transition. We then show that the true counterfactual
transition of $\mathcal M$ is compatible, and that the corresponding
solution is the true counterfactual occupancy~\eqref{eq:ctf-occupancy}.
Averaging over the demonstrator's action then recovers the true
occupancy, and minimizing over $\mathbb T$ can only decrease it.

For every $\mathcal T\in\mathbb T$, Eq.~\eqref{eq:ctf-bellman} is the
fixed-point equation
$q=\mathcal K^\pi_{\mathcal T}q$. By Lemma~\ref{lem:bellman}(a), it
has a unique bounded solution $Q^\pi_{\mathcal T}(s,a,a')$.

We next show that the true counterfactual transition of $\mathcal M$
is compatible. Let
\[
\mathcal T_{\mathcal M}(s,a,a',s')
=
P\big(S^{(t+1)}_{A^{(t)}=a}=s'\mid S^{(t)}=s,A^{(t)}=a'\big)
\]
denote the counterfactual transition~\eqref{eq:ctf-transition} of
$\mathcal M$. On the event $\{A^{(t)}=a\}$, the consistency axiom gives
$S^{(t+1)}_{A^{(t)}=a}=S^{(t+1)}$. Hence
\[
\begin{aligned}
\mathcal T_{\mathcal M}(s,a,a,s')
&=
P\big(S^{(t+1)}=s'\mid S^{(t)}=s,A^{(t)}=a\big)
\\
&=
P(s'\mid s,a),
\end{aligned}
\]
so $\mathcal T_{\mathcal M}\in\mathbb T$.

We now show that the true counterfactual
occupancy~\eqref{eq:ctf-occupancy} solves Eq.~\eqref{eq:ctf-bellman}
with $\mathcal T_{\mathcal M}$. Fix $(s,a,a')$ with $P(a'\mid s)>0$, and let
$E=\{S^{(t)}=s,A^{(t)}=a'\}$. The $k=0$ term of
Eq.~\eqref{eq:ctf-occupancy} equals $(1-\gamma)\mathbbm{1}\{s=s^*\}$,
since $S^{(t)}_{A^{(t)}=a,\pi}=S^{(t)}=s$. For $k\geq1$, conditioning on
the counterfactual next state gives
\[
\begin{aligned}
&P\big(S^{(t+k)}_{A^{(t)}=a,\pi}=s^*\mid E\big)
\\
&\;=
\sum_{s'}\mathcal T_{\mathcal M}(s,a,a',s')\,
P\big(S^{(t+k)}_{A^{(t)}=a,\pi}=s^*\mid E_{s'}\big),
\end{aligned}
\]
where $E_{s'}=E\cap\{S^{(t+1)}_{A^{(t)}=a}=s'\}$. Since
$A^{(t)}=f_A(S^{(t)},U^{(t)})$ and
$S^{(t+1)}_{A^{(t)}=a}=f_S(S^{(t)},a,U^{(t)})$, the event $E_{s'}$ is
determined by $(S^{(t)},U^{(t)})$. From time $t+1$ onward, the
counterfactual trajectory starts at $s'$, draws
$\tilde a\sim\pi(\cdot\mid s')$, and then evolves through $f_S$ and
$\pi$, driven by $U^{(t+1)},U^{(t+2)},\dots$, which are drawn
independently of $(S^{(t)},U^{(t)})$. Conditional on $E_{s'}$, it
therefore has the same law as a trajectory that starts at
$S^{(t+1)}=s'$, sets $A^{(t+1)}$ to $\tilde a\sim\pi(\cdot\mid s')$, and
follows $\pi$ thereafter. Since the structural mechanisms are invariant
over time, it follows that
\[
\begin{aligned}
&(1-\gamma)\sum_{k\geq1}\gamma^{k}\,
P\big(S^{(t+k)}_{A^{(t)}=a,\pi}=s^*\mid E_{s'}\big)
\\
&\;=
\gamma\sum_{\tilde a}\pi(\tilde a\mid s')\,
d^\pi(s^*\mid s',\tilde a).
\end{aligned}
\]
Applying the law of total probability to the demonstrator's action at
$s'$, as below Eq.~\eqref{eq:ctf-occupancy}, gives
\[
d^\pi(s^*\mid s',\tilde a)
=
\sum_{\tilde a'}P(\tilde a'\mid s')\,
Q^\pi(s',\tilde a,\tilde a').
\]
Combining the three displays, we obtain
\[
\begin{aligned}
Q^\pi(s,a,a')
&=
(1-\gamma)\mathbbm{1}\{s=s^*\}
\\
&\quad+
\gamma\sum_{s'}\mathcal T_{\mathcal M}(s,a,a',s')
\sum_{\tilde a,\tilde a'}
\pi(\tilde a\mid s')
\\
&\qquad\times
P(\tilde a'\mid s')\,
Q^\pi(s',\tilde a,\tilde a'),
\end{aligned}
\]
that is,
$Q^\pi=\mathcal K^\pi_{\mathcal T_{\mathcal M}}Q^\pi$. By uniqueness of this solution,
$Q^\pi_{\mathcal T_{\mathcal M}}(s,a,a')=Q^\pi(s,a,a')$.

Finally, averaging over $a'\sim P(\cdot\mid s)$ and applying the law of
total probability as below Eq.~\eqref{eq:ctf-occupancy} gives
\[
\sum_{a'}Q^\pi_{\mathcal T_{\mathcal M}}(s,a,a')\,P(a'\mid s)
=
Q^\pi(s,a).
\]
Since $\mathcal T_{\mathcal M}\in\mathbb T$, the infimum
defining $\underline Q^\pi(s,a)$ is at most its value at
$\mathcal T_{\mathcal M}$, and hence
\[
\underline Q^\pi(s,a)
\leq
Q^\pi(s,a)
\quad\text{for every }(s,a).
\]
This completes the proof.
\end{proof}

\begin{proof}[Proof of Proposition~\ref{prop:ctf-critic}]
The proof has two steps. We first show that the positive goals
generated by \textsc{CtfRollout} follow the goal-reaching occupancy of
$\pi$ in $\mathcal M_{\mathcal T}$; the claim then follows by computing
the Bayes-optimal critic of Eq.~\eqref{eq:caucrl-nce} and identifying
it through Lemma~\ref{lem:bellman}(c). Throughout, we consider
untruncated rollouts, i.e., $H=\infty$ (the effect of truncation is
discussed after the proposition), and optimize
Eq.~\eqref{eq:caucrl-nce} over all functions $f$. The critic is then
determined on pairs $(s,a)$ in the support of the distribution $\nu$ of
training pairs drawn from $\mathcal D^{+}$, and on goals with
$P^-(s^*)>0$.

Consider the iteration $h$ of \textsc{CtfRollout}. The demonstrator's
action
$a^{\prime(h)}$ is drawn from $\mu(\cdot\mid\tilde s^{(h)})$
independently of the past given $\tilde s^{(h)}$. Marginalizing it out
and using $\mu=P(\cdot\mid s)$, the next state satisfies
\[
\begin{aligned}
&P\big(\tilde s^{(h+1)}=s'\mid
\tilde s^{(t:h)},\tilde a^{(t:h)}\big)
\\
&\;=
\sum_{a'}
\mathcal T(\tilde s^{(h)},\tilde a^{(h)},a',s')\,
P(a'\mid\tilde s^{(h)})
\\
&\;=
P^{\mathcal T}_{\tilde a^{(h)}}(s'\mid\tilde s^{(h)}),
\end{aligned}
\]
by Eq.~\eqref{eq:induced-transition}, and the next action is drawn from
$\pi(\cdot\mid\tilde s^{(h+1)})$. Hence the rollout from $(s,a)$ has the
trajectory distribution $P^{\mathcal T}_{a,\pi}$. Since $K$ is drawn
independently with $P(K=k)=(1-\gamma)\gamma^k$, the positive goal
$\tilde s^+=\tilde S^{(t+K)}$ satisfies
\[
\begin{aligned}
P(\tilde s^+=s^*\mid s,a)
&=
(1-\gamma)\sum_{k=0}^{\infty}\gamma^k
\\
&\quad\times
P^{\mathcal T}_{a,\pi}\big(S^{(t+k)}=s^*\mid S^{(t)}=s\big)
\\
&=
d^\pi_{\mathcal T}(s^*\mid s,a).
\end{aligned}
\]

We now compute the critic that maximizes Eq.~\eqref{eq:caucrl-nce}. By
the previous display, positive triples $(s,a,\tilde s^+)$ are distributed as
$q^+(s,a,s^*)=\nu(s,a)\,d^\pi_{\mathcal T}(s^*\mid s,a)$. Since $s^-$
is drawn independently of $(s,a)$, negative triples $(s,a,s^-)$ are
distributed as $q^-(s,a,s^*)=\nu(s,a)\,P^-(s^*)$. The objective in
Eq.~\eqref{eq:caucrl-nce} therefore equals
\[
\sum_{s,a,s^*}
\Big[
q^+\log\sigma(f)
+
q^-\log\big(1-\sigma(f)\big)
\Big],
\]
which decouples across $(s,a,s^*)$. For each triple with $\nu(s,a)>0$
and $P^-(s^*)>0$, the map
$x\mapsto q^+\log\sigma(x)+q^-\log(1-\sigma(x))$ is strictly concave
and is maximized at $\sigma(x)=q^+/(q^++q^-)$, i.e., at
$\exp x=q^+/q^-$, with the convention $\exp x=0$ when $q^+=0$. The
factor $\nu(s,a)$ cancels, so the optimal critic satisfies
\[
\exp f^*(s,a,s^*)
=
\frac{d^\pi_{\mathcal T}(s^*\mid s,a)}{P^-(s^*)}.
\]
By Lemma~\ref{lem:bellman}(c),
$d^\pi_{\mathcal T}(s^*\mid s,a)=Q^\pi_{\mathcal T}(s,a)$, and therefore
\[
\exp f^*(s,a,s^*)
=
\frac{1}{P^-(s^*)}\cdot Q^\pi_{\mathcal T}(s,a).
\]
This proves the claim.
\end{proof}

\begin{proof}[Proof of Proposition~\ref{prop:policy-improvement}]
The proof has two steps and follows the approximate policy improvement
argument of \citet[Lemma~4.2]{NEURIPS2022_e7663e97}, with
$\mathcal M_{\mathcal T}$ in place of the true environment. We first
show that the greedy policy $\pi'$ is $2\delta$-greedy with respect to
$Q^\pi_{\mathcal T}$; the claim then follows by propagating this
inequality through the Bellman operator of $\pi'$.

By Lemma~\ref{lem:bellman}(c), $Q^\pi_{\mathcal T}$ and
$Q^{\pi'}_{\mathcal T}$ are the fixed points of
$\mathcal B^\pi_{\mathcal T}$ and $\mathcal B^{\pi'}_{\mathcal T}$,
respectively. Let
$\widehat Q(s,a)=P^-(s^*)\exp f(s,a,s^*)$, so that
$|\widehat Q-Q^\pi_{\mathcal T}|\leq\delta$ pointwise by assumption, and
let $V^\pi(s)=\sum_a\pi(a\mid s)\,Q^\pi_{\mathcal T}(s,a)$.

Since $P^-(s^*)>0$, the map $x\mapsto P^-(s^*)\exp x$ is increasing, so
$\pi'(s)\in\arg\max_af(s,a,s^*)$ also maximizes $\widehat Q(s,\cdot)$.
Hence, for every $s$,
\[
\begin{aligned}
Q^\pi_{\mathcal T}(s,\pi'(s))
&\geq
\widehat Q(s,\pi'(s))-\delta
\\
&\geq
\sum_a\pi(a\mid s)\,\widehat Q(s,a)-\delta
\\
&\geq
V^\pi(s)-2\delta,
\end{aligned}
\]
where the first and last inequalities use
$|\widehat Q-Q^\pi_{\mathcal T}|\leq\delta$, and the middle one uses
that a maximum dominates any average.

We next propagate this bound. Since $\pi'$ is deterministic,
\[
\begin{aligned}
(\mathcal B^{\pi'}_{\mathcal T}Q^\pi_{\mathcal T})(s,a)
&=
(1-\gamma)\mathbbm{1}\{s=s^*\}
\\
&\quad+
\gamma\sum_{s'}P^{\mathcal T}_a(s'\mid s)\,
Q^\pi_{\mathcal T}(s',\pi'(s')).
\end{aligned}
\]
Applying the previous inequality at every $s'$ and using that $Q^\pi_{\mathcal T}$ is the
fixed point of $\mathcal B^\pi_{\mathcal T}$, we obtain
\[
\mathcal B^{\pi'}_{\mathcal T}Q^\pi_{\mathcal T}
\geq
\mathcal B^{\pi}_{\mathcal T}Q^\pi_{\mathcal T}-2\gamma\delta
=
Q^\pi_{\mathcal T}-2\gamma\delta
\]
pointwise. We now show by induction that
\[
(\mathcal B^{\pi'}_{\mathcal T})^{n}Q^\pi_{\mathcal T}
\geq
Q^\pi_{\mathcal T}-2\gamma\delta\sum_{k=0}^{n-1}\gamma^k
\]
for every $n\geq1$. The case $n=1$ is the display above. If the bound
holds for $n$, then the monotonicity of $\mathcal B^{\pi'}_{\mathcal T}$
and the identity $\mathcal B^{\pi'}_{\mathcal T}(q-c)
=\mathcal B^{\pi'}_{\mathcal T}q-\gamma c$ from
Lemma~\ref{lem:bellman}(a) give
\[
\begin{aligned}
(\mathcal B^{\pi'}_{\mathcal T})^{n+1}Q^\pi_{\mathcal T}
&\geq
\mathcal B^{\pi'}_{\mathcal T}Q^\pi_{\mathcal T}
-2\gamma^2\delta\sum_{k=0}^{n-1}\gamma^k
\\
&\geq
Q^\pi_{\mathcal T}-2\gamma\delta\sum_{k=0}^{n}\gamma^k.
\end{aligned}
\]
Letting $n\to\infty$, Lemma~\ref{lem:bellman}(a) gives
$(\mathcal B^{\pi'}_{\mathcal T})^{n}Q^\pi_{\mathcal T}\to
Q^{\pi'}_{\mathcal T}$, and therefore
\[
Q^{\pi'}_{\mathcal T}(s,a)
\geq
Q^{\pi}_{\mathcal T}(s,a)-\frac{2\gamma\delta}{1-\gamma}
\quad\text{for all }(s,a).
\]
This completes the proof.
\end{proof}

\section{Implementation Details}
\label{app:implementation}

This section provides implementation details for CauCRL beyond the
algorithmic description in Sec.~3. The continuous-control experiments use
environment-specific parameterizations of the nominal behavioral policy and
counterfactual transition model. Network architectures and optimization
hyperparameters are reported in Appendix~\ref{app:hyperparameters}.

\subsection{Nominal Behavioral Policy}
\label{app:nominal-policy}

For both continuous-control environments, we parameterize the nominal
behavioral policy as a $K$-component censored Gaussian mixture density
network with $K=5$ and diagonal covariance,
\[
    \mu_\phi(a\mid c)
    =
    \sum_{k=1}^{K}
    w_k(c)\,
    \mathcal{N}^{\mathrm{cens}}_{\mathcal A}
    \bigl(a;\mu_k(c),\sigma_k(c)\bigr),
\]
where $\mathcal{N}^{\mathrm{cens}}_{\mathcal A}$ denotes the Gaussian
distribution censored to the action box $\mathcal A$, with continuous
density in the interior and point mass at clipped boundaries. The
conditioning context is $c=s$ for AntMaze and $c=(s,g)$ for PointMaze.
The PointMaze context contains only learner-visible information;
privileged hazard and demonstrator metadata are excluded.

For Windy Corridor, the action space is discrete, and we parameterize the
nominal behavioral policy as a categorical distribution with context
$c=s$,
\[
\begin{aligned}
    \mu_\phi(a\mid c)
    &=
    \operatorname{softmax}\bigl(f_\phi(c)\bigr)_a,
    \\
    a &\in \{\texttt{left}, \texttt{right}, \texttt{forward}, \texttt{wait}\},
\end{aligned}
\]
where $f_\phi$ is an MLP with two hidden layers of 128 ReLU units applied
to a one-hot encoding of the agent's $x$-coordinate, $y$-coordinate, and
heading. The \texttt{wait} action is the demonstrator's no-op, which it
uses to pause when the wind is unfavorable. It is kept in the support of
$\mu_\phi$ because it occurs in the demonstrations, even though the learned
policy selects only among the three movement actions.

In all three environments, the nominal-policy likelihood is
\begin{equation}
    \mathcal{L}_{\mu}(\phi)
    =
    -\mathbb{E}_{(c,a)\sim\mathcal{D}_{\mu}}
    \left[
        \log \mu_\phi(a\mid c)
    \right].
    \label{eq:nominal-policy-loss}
\end{equation}
The environments use different fitting procedures. In PointMaze,
$\mu_\phi$ is trained separately on the $4{,}800$ expert episodes,
excluding the uniform-coverage and blind-demonstrator subsets, and the
checkpoint is selected on a validation split before the model is frozen.
In AntMaze, $\mathcal{L}_{\mu}$ is optimized jointly with the
counterfactual-transition objective with coefficient $\lambda_\mu=1$, and
the final iterate is used. In Windy Corridor, $\mathcal{D}_{\mu}$ consists
of the $1{,}000$ demonstrations in $\mathcal{D}^{+}$, on which
$\mathcal{L}_{\mu}$ reduces to the four-way cross-entropy. We optimize it
separately with Adam (learning rate $10^{-3}$, batch size 256) for
$3{,}000$ updates and use the final iterate. Because the state space is
small ($15\times15$ cells and four headings), we evaluate the frozen
$\mu_\phi$ once on every reachable state and store the resulting table,
from which the counterfactual rollouts sample the nominal action.

For AntMaze, the learned mixture weights are also used by the rollout
construction in Appendix~\ref{app:rollout-implementation}. Network and
optimization hyperparameters are reported in
Appendix~\ref{app:hyperparameters}.

\subsection{Counterfactual Transition Model}
\label{app:ctf-transition-model}

The three environments use different parameterizations of the
counterfactual transition model $T_\eta(s,a,a',\cdot)$ introduced in
Sec.~3.1, where $a$ denotes the target action and $a'$ the demonstrator
action. Windy Corridor uses a finite categorical model with locally
restricted support, while PointMaze and AntMaze use neural transition
models with different spectral-normalization schemes.

\textbf{Windy Corridor.}
Because the state and action spaces are finite, the counterfactual
transition is represented explicitly over states and action pairs.
Actions other than \texttt{forward} have deterministic one-step dynamics
and are represented by point masses at their corresponding successors.
Factual forward transitions use the empirical one-step distribution
$\widehat P(\cdot\mid s,a)$.

For off-diagonal forward transitions, we use a categorical kernel
$\kappa_\eta$ to combine the empirical transition distributions associated
with the two actions:
\begin{equation}
\begin{aligned}
T_\eta(s,a,a',y)
&=
(1-\epsilon)
\sum_{x,x'}
\widehat P(x\mid s,a)\,
\widehat P(x'\mid s,a')
\\
&\quad\times
\kappa_\eta(y\mid x,x',s,a,a')
+
\frac{\epsilon}{|\mathcal S|},
\end{aligned}
\label{eq:windy-ctf-transition}
\end{equation}
with $\epsilon=10^{-3}$.

The kernel is categorical over $75$ local offsets around the deterministic
successor of the target action, corresponding to
$\Delta x,\Delta y\in\{-2,\ldots,2\}$ and heading changes in
$\{0,\pm1\}$. Its support is restricted to candidate successors within
radius $r=2$ of both factual anchors under the state metric used by the
model. The kernel is parameterized using translation-invariant local
features and does not receive the latent wind variable. It is trained by
cross-entropy using the counterfactual supervision described in
Sec.~3.1.

\textbf{PointMaze.}
PointMaze uses a stochastic generator that predicts the two-dimensional
planar displacement of the agent conditioned on the current state, target
action $a$, and demonstrator action $a'$. Let $\Delta x=x'-x$ denote the
observed planar displacement. For each input, the model produces $M=8$
samples $\Delta\tilde{x}_1,\ldots,\Delta\tilde{x}_M$ and is trained with
the Monte Carlo energy score
\begin{equation}
\begin{aligned}
    \widehat{\ell}_{\mathrm{ES}}
    &=
    \frac{1}{M}
    \sum_{m=1}^{M}
    \|\Delta\tilde{x}_m-\Delta x\|_2
    \\
    &\quad
    -
    \frac{1}{2M(M-1)}
    \sum_{m\neq n}
    \|\Delta\tilde{x}_m-\Delta\tilde{x}_n\|_2 .
\end{aligned}
\label{eq:pointmaze-energy-score}
\end{equation}
The sampled displacement is clipped before constructing the next stacked
state by shifting the position history and inserting the generated planar
position.

To control sensitivity to the target action, the linear maps on the
target-action pathway are projected to spectral norm at most $1$.
Together with the $1$-Lipschitz nonlinearities and clipping operation,
this gives a Lipschitz bound of $1$ with respect to the target action in
raw action units. An auxiliary classifier is trained jointly to predict the
hazard-transition indicator. Let $\mathcal{D}_{\mathrm{diag}}$ and
$\mathcal{D}_{\mathrm{off}}$ denote the diagonal and off-diagonal
subsets of a data partition $\mathcal{D}$. For each regime
$r\in\{\mathrm{diag},\mathrm{off}\}$, we define
\begin{equation}
    \mathcal{L}_{r}(\mathcal{D})
    =
    \frac{1}{|\mathcal{D}_{r}|}
    \sum_{i\in\mathcal{D}_{r}}
    \widehat{\ell}_{\mathrm{ES},i}
    +
    \frac{\lambda_{\mathrm{on}}}{|\mathcal{D}|}
    \sum_{i\in\mathcal{D}_{r}}
    \ell^{\mathrm{BCE}}_{\mathrm{onset},i},
\end{equation}
where $\lambda_{\mathrm{on}}=0.25$. The complete PointMaze
transition-model objective is
\begin{equation}
    \mathcal{L}_{\mathrm{ETT}}(\mathcal{D})
    =
    \mathcal{L}_{\mathrm{diag}}(\mathcal{D})
    +
    \mathcal{L}_{\mathrm{off}}(\mathcal{D}).
\end{equation}
The onset component is estimated using class-stratified sampling and
reweighted to the empirical event prevalence.

\textbf{AntMaze.}
AntMaze uses a deterministic predictor of the full state displacement,
conditioned on the current state, target action $a$, and demonstrator
action $a'$. Its predictive distribution is a Dirac measure, for which
the energy score reduces to
\begin{equation}
    \ell_{\mathrm{ES}}
    =
    \|g_\eta(s,a,a')-(s'-s)\|_2 .
    \label{eq:antmaze-energy-score}
\end{equation}
Spectral normalization is applied to every layer of the displacement
network using $10$ power iterations and coefficient $2.0$. This
coefficient controls the per-layer spectral scale and is used as a
regularizer; it is not interpreted as a global Lipschitz constant of the
transition network.

Additional architecture and optimization hyperparameters are reported in
Appendix~\ref{app:hyperparameters}.

\subsection{Counterfactual Rollout and Replay Construction}
\label{app:rollout-implementation}

Model-generated continuations are stored separately from the original
offline demonstrations, which remain fixed throughout training. Generated
trajectories preserve their temporal structure so that future-goal
relabeling is performed within the same recorded or generated trajectory.

\textbf{Windy Corridor.}
For Windy Corridor, each continuation starts from a recorded
state--action pair $(s,a)$ drawn uniformly from $\mathcal D^{+}$, and the
recorded action is kept as the first target action. At each generated
state, a nominal demonstrator action $a'$ is sampled from the tabulated
$\mu_\phi$, the next state is sampled from the tabulated
$T_\eta(s,a,a',\cdot)$, and the next target action is sampled from the
current learner. Because the target action always comes from the learner,
the factual branch is used exactly when the learner's action coincides
with the sampled demonstrator action, and no separate retention
probability is needed. A continuation ends at lava, at the goal, or at
the rollout horizon. The augmentation buffer is regenerated with the
current policy at the start of each training round. Contrastive positives
are relabeled from future states of the generated continuations, while
behavioral cloning continues to use the recorded actions in
$\mathcal D^{+}$.

\textbf{PointMaze.}
For PointMaze, we retain a recorded state--action pair and replace its
subsequent continuation with a rollout from the learned counterfactual
transition model. The generated suffix supplies future states for
contrastive goal relabeling, while behavioral-cloning supervision continues
to use the recorded actions from the offline demonstrations.

\textbf{AntMaze.}
For AntMaze, the recorded prefix is retained up to the route decision and
the remaining trajectory is replaced by a model-generated continuation.
At each generated state, a nominal demonstrator action is sampled from
$\mu_\phi$ and a target action is sampled from the current learner.
The probability of retaining the factual branch is
\[
    p_{\mathrm{diag}}(s)=\sum_{k=1}^{K} w_k(s)^2,
\]
computed from the mixture weights of the nominal behavioral policy;
otherwise the transition model is evaluated using the learner-selected
action.

We generate continuations of length $90$ and refresh the augmentation
buffer throughout actor--critic training using the current policy.
Each refresh contains $2{,}000$ generated continuations, and generated
and recorded trajectories each contribute half of the training minibatch.
Additional rollout parameters are listed in
Appendix~\ref{app:hyperparameters}.

\subsection{Contrastive Critic and Actor Training}
\label{app:crl-implementation}

CauCRL uses the same actor--critic parameterization as the corresponding
CRL baseline within each environment. PointMaze represents goals using the
full stacked state, whereas AntMaze uses the planar $(x,y)$ position.

The implementation uses in-batch negative sampling for contrastive critic
training. PointMaze uses a single critic, while AntMaze uses two
independently parameterized critics and evaluates actor actions using the
minimum of their scores.

The actor is a tanh-Gaussian goal-conditioned policy. Behavioral-cloning
regularization uses recorded offline actions; PointMaze draws these actions
from a separately balanced BC minibatch, whereas AntMaze uses the standard
training minibatch. Network dimensions and optimization hyperparameters are
reported in Appendix~\ref{app:hyperparameters}.

\section{Experimental Setup}
\label{app:experimental-setup}

\subsection{Windy Corridor}
\label{app:windy-setup}

\textbf{Geometry.}
Windy Corridor is a $15\times15$ MiniGrid~\citep{chevalier2023minigrid} environment with two routes from
the start $(1,1)$ to the goal $(13,1)$. The short route runs east along
row $y=1$, directly above a row of lava cells at $x\in\{3,4,5\}$ and
$x\in\{8,9,10\}$. The safe route descends along column $x=1$, follows row
$y=13$, and returns north along column $x=13$; it takes about $39$ steps,
compared with about $13$ on the short route. Entering a lava cell
terminates the episode as a failure.

\textbf{Dynamics and wind.}
The agent's actions are \texttt{left}, \texttt{right}, and
\texttt{forward}; the demonstrator can additionally \texttt{wait}. A wind
direction is drawn independently at every step, east, south, west, or
north with probability $0.1$ each and still with probability $0.6$. It
is observed by the demonstrator before acting and resampled after the
transition. Wind affects only \texttt{forward}: a tailwind moves the agent
two cells, a headwind cancels the move, and a crosswind moves it one cell
forward and then one cell in the wind direction. On the short route, a
south wind therefore pushes the agent into the lava from six positions
along row $y=1$; the safe route has no such positions.

\textbf{Observation.}
The learner observes its cell and heading, $s_t=(x_t,y_t,h_t)$, and the
goal is the target cell $(x,y)$. The wind is never included in the
learner observation.

\textbf{Demonstrator and dataset.}
At the beginning of each episode, the demonstrator selects the short route
with probability $0.95$ and the safe route otherwise. On the short route,
it observes the current wind and waits whenever it stands at one of the
six exposed positions under a south wind; on the safe route, it follows
the corridor without consulting the wind. The demonstration dataset
$\mathcal D^{+}$ contains $1{,}000$ episodes and $14{,}028$ transitions,
of which $952$ follow the short route and $48$ the safe route. Every
episode reaches the goal, so the dataset contains no lava failures.
All offline RL methods use $\mathcal D^{+}$.

A separate pool of $1{,}000$ trajectories of a uniformly random policy is
used only to construct $P^{-}$ for the counterfactual transition model
(Appendix~\ref{app:ctf-transition-model}) and is excluded from
actor--critic replay.

\textbf{Evaluation.}
Each evaluation checkpoint uses $300$ episodes with a fixed set of
environment seeds, deterministic policy actions, and a horizon of $400$
steps under the natural wind distribution. Success is recorded when the
agent reaches the goal cell. We additionally report the robust-route
rate, the fraction of episodes that enter the safe-route row $y=13$.
\subsection{PointMaze}
\label{app:pointmaze-setup}

\textbf{Geometry.}
PointMaze is a $9\times5$ grid maze with two start-to-goal routes. The
agent starts at $(0.5,3.5)$ and the goal is at $(8.5,3.5)$. The short
route runs along row $y=3$ and crosses three hazard cells at
$(3,3)$, $(4,3)$, and $(5,3)$. The safe route branches downward at
$x=1$, follows row $y=1$, and rejoins the short route at $x=7$.
Episodes have a horizon of $50$ steps.

\textbf{Dynamics and hazard.}
Actions are two-dimensional velocities clipped to $[-1,1]^2$, with
additive Gaussian noise of standard deviation $0.01$. Each environment
step is integrated over $10$ simulation substeps with per-axis wall
collision.

Each of the three hazard cells is independently active with probability
$0.30$ at every step. The hazard indicators are observed by the
demonstrator before acting and are resampled after the transition. If the
agent ends a step inside an active hazard cell, it enters an absorbing
failure state: its position remains fixed, subsequent actions have no
effect, and the goal reward can no longer be obtained.

\textbf{Observation.}
The learner observes the four most recent planar positions,
\[
    s_t =
    [x_t,y_t,\ldots,x_{t-3},y_{t-3}],
\]
giving an $8$-dimensional state representation. The goal is represented
in the corresponding stacked form. Hazard indicators are never included
in the learner observation.

\textbf{Demonstrator and dataset.}
At the beginning of each episode, the demonstrator selects either the
safe route or the short route. Safe-route demonstrations follow a
shortest path computed with the hazard cells treated as blocked. On the
short route, the demonstrator observes the current hazard indicators and,
before entering a hazard cell, avoids actions predicted to terminate in
an active cell.

The offline dataset contains $6{,}600$ episodes and $330{,}000$
transitions:
\begin{center}
\small
\begin{tabular}{lr}
\toprule
Behavior & Episodes \\
\midrule
Wait at the hazard, then cross & $3{,}073$ \\
Cross immediately              & $1{,}470$ \\
Uniform coverage               & $1{,}200$ \\
Blind demonstrator             & $600$ \\
Forced safe detour             & $257$ \\
\bottomrule
\end{tabular}
\end{center}
The resulting dataset contains $4{,}514$ successful episodes,
$1{,}451$ absorbed failures, and $635$ timeouts. All offline RL methods
use the complete dataset. The nominal behavioral policy described in
Appendix~\ref{app:nominal-policy} is fitted only on the $4{,}800$
demonstrator episodes, excluding the uniform-coverage and
blind-demonstrator subsets.

\textbf{Evaluation.}
Each evaluation checkpoint uses $300$ episodes with a fixed set of
environment seeds and deterministic policy actions. Success is recorded
when the agent reaches the task goal according to the environment reward.
We additionally report the safe-detour rate. An episode is classified as
taking the detour if it enters the lower route before entering any hazard
cell; otherwise, entering a hazard cell first is classified as taking the
shortcut.

\subsection{AntMaze}
\label{app:antmaze-setup}

\textbf{Geometry.}
AntMaze is a continuous-control ring maze with two routes from the start
to the goal. The Ant starts at $(0,0)$ and the goal is at $(24,0)$.
The direct route crosses two stochastic rockfall zones, while a longer
perimeter route avoids both. Episodes have a horizon of $800$ steps.

The two rockfall zones lie on the direct route. Zone~1 spans
$x\in[6.6,9.4]$ and Zone~2 spans $x\in[14.6,17.4]$, with both covering
$y\in[-2,2]$. The learner observes a $29$-dimensional Ant state and
selects $8$-dimensional joint torques; the goal is represented by its
planar $(x,y)$ position. Rockfall states and activation schedules are
hidden from the learner.

\textbf{Hazard.}
Each rockfall zone is independently activated with probability
$p_{\mathrm{active}}=0.40$ at the beginning of an episode. Conditional on
activation, the start times are sampled independently as
\[
    t_0^{(1)}\sim\mathcal{U}\{10,\ldots,45\},
    \qquad
    t_0^{(2)}\sim\mathcal{U}\{90,\ldots,125\}.
\]
An activated zone remains active for $72$ steps and drops four rocks.
Physical contact between a rock and the Ant produces an absorbing failure:
the agent loses control for the remainder of the episode and can no longer
receive the goal reward.

\textbf{Demonstrator and dataset.}
The demonstrator first chooses between the direct and detour routes using
an episode-level route variable independent of the rockfall activations.
Detour trajectories follow the safe perimeter route without consulting the
rockfall schedules. On direct-route trajectories, the demonstrator observes
the activation schedule of each zone and decides whether to enter or wait.
A crossing is attempted only when the predicted crossing interval, using a
$50$-step crossing estimate and an $8$-step safety margin, does not overlap
the corresponding rockfall interval.

The demonstration dataset contains $1{,}000$ episodes and $311{,}005$
transitions, collected with $p_{\mathrm{active}}=0.40$. Direct and detour
routes each account for $50\%$ of the demonstrations. The dataset contains
$994$ successful episodes and $6$ timeouts, with no rock-contact failures.

A separate pool of $300$ episodes is used only for fitting the
counterfactual transition model and is excluded from actor--critic replay.
This pool is collected in the same environment by overriding a
demonstrator decision to wait at Zone~1 with a go decision and retaining
trajectories in which the override produces a realized Zone~1 rock-contact
failure.

\textbf{Evaluation.}
Each evaluation checkpoint uses $300$ episodes with fixed environment
seeds and deterministic policy actions over the full $800$-step horizon.
Evaluation uses the same rockfall activation probability
$p_{\mathrm{active}}=0.40$ as data collection. Success is recorded when
the Ant reaches the task goal at $(24,0)$ according to the environment
reward.

\subsection{Network Architectures and Training Hyperparameters}
\label{app:hyperparameters}

Table~\ref{tab:agent-hparams} summarizes the actor--critic configurations
used for all three environments. Within each environment, CRL and CauCRL use
the same actor and critic architectures. All experiments were run on
NVIDIA GeForce RTX 5090 GPUs.

\begin{table}[t]
\centering
\small
\caption{Actor--critic architectures and training hyperparameters.}
\label{tab:agent-hparams}
\setlength{\tabcolsep}{4pt}

\resizebox{\columnwidth}{!}{%
\begin{tabular}{lccc}
\toprule
 & \textbf{Windy Corridor} & \textbf{PointMaze} & \textbf{AntMaze} \\
\midrule
State dimension & 3 & 8 & 29 \\
Action dimension & 3 (discrete) & 2 & 8 \\
Goal dimension & 2 & 8 & 2 \\
Episode horizon & 400 & 50 & 800 \\
Batch size & 64 & 256 & 1024 \\
Actor hidden layers & $[128]$ & $[256,256]$ & $[1024,1024]$ \\
Critic hidden layers & $[128]$ & $[256,256]$ & $[1024,1024]$ \\
Activation & ReLU & ReLU & ReLU \\
Representation dimension & 64 & 64 & 16 \\
Number of critics & 1 & 1 & 2 \\
Discount $\gamma$ & 0.99 & 0.95 & 0.999 \\
BC coefficient $\lambda_{\mathrm{BC}}$ & 1 & 0.05 & 0.05 \\
Entropy coefficient & 0 & 0 & 0 \\
Actor learning rate & $2\times10^{-3}$ & $3\times10^{-4}$ & $3\times10^{-4}$ \\
Critic learning rate & $2\times10^{-3}$ & $3\times10^{-4}$ & $3\times10^{-4}$ \\
Optimizer & SGD & Adam & Adam \\
Main updates & 100k & 300k & 100k \\
Critic warm-up & -- & 30k & -- \\
Evaluation interval & 250 & 1k & 400 \\
Evaluation episodes & 300 & 300 & 300 \\
\bottomrule
\end{tabular}%
}
\end{table}

For Windy Corridor, the critic encodes the state--action pair and the goal
with one-hot features over grid states and cells, and the contrastive score
uses temperature $0.07$. The actor receives the raw $(x,y,h)$ state. Its
objective $(1-\lambda)\,\mathbb{E}_{\pi}[f]+\lambda\log\pi(a\mid s,g)$
with $\lambda=0.5$ corresponds to $\lambda_{\mathrm{BC}}=1$. Actor
minibatches are drawn uniformly over distinct recorded state--action pairs
rather than over transitions, so the rare safe-route decisions are not
outweighed by the frequent short-route ones.

Table~\ref{tab:nominal-hparams} reports the nominal behavioral-policy
configurations. As described in Appendix~\ref{app:nominal-policy}, the
PointMaze and Windy Corridor models are trained separately and then frozen,
whereas the AntMaze model is optimized jointly with the counterfactual
transition model.

\begin{table}[t]
\centering
\small
\caption{Nominal behavioral-policy architectures and training settings.}
\label{tab:nominal-hparams}

\setlength{\tabcolsep}{3pt}
\begin{tabularx}{\columnwidth}{
    @{}
    >{\raggedright\arraybackslash}p{0.27\columnwidth}
    >{\centering\arraybackslash}X
    >{\centering\arraybackslash}X
    >{\centering\arraybackslash}X
    @{}
}
\toprule
 & \textbf{Windy Corridor} & \textbf{PointMaze} & \textbf{AntMaze} \\
\midrule
Conditioning context
    & state (one-hot, 34-D)
    & state + goal (16-D)
    & state (29-D) \\

Action dimension
    & 4 (discrete)
    & 2
    & 8 \\

Mixture components $K$
    & -- (categorical)
    & 5
    & 5 \\

Hidden layers
    & $[128,128]$
    & $[128,128]$
    & $[512,512]$ \\

Activation
    & ReLU
    & SiLU
    & ReLU \\

Covariance
    & --
    & diagonal
    & diagonal \\

Action likelihood
    & categorical
    & censored on $[-1,1]^2$
    & censored on empirical action box \\

Scale parameterization
    & --
    & $\sigma\in[0.02,2.0]$, $|\mu|\leq3$
    & $\log\sigma\in[-5,2]$ \\

Spectral normalization
    & none
    & none
    & all layers, coef.\ $2.0$, 10 power iterations \\

Batch size
    & 256
    & 512
    & 1024 \\

Learning rate
    & $10^{-3}$
    & $3\times10^{-4}$
    & $10^{-3}$ \\

Gradient clipping
    & --
    & 10.0
    & -- \\

Training updates
    & 3k
    & 20k
    & 20k \\

Optimization
    & separate, then frozen
    & separate, then frozen
    & joint with transition model,
      $\lambda_\mu=1$ \\

Checkpoint selection
    & final iterate
    & best validation NLL
    & final iterate \\

Training data
    & 900 expert episodes
    & 4,800 expert episodes
    & full demonstration pool \\

Validation split
    & --
    & 10\% by episode
    & episode-level holdout shared with the counterfactual transition model \\
\bottomrule
\end{tabularx}
\end{table}

For PointMaze, the $4{,}800$ nominal-policy episodes are split by complete
episode into $4{,}321$ training and $479$ validation episodes. Validation
NLL is evaluated every $500$ updates. Context normalization statistics are
fit on the training split only. For AntMaze, the nominal-policy head uses
the same positive-sample normalization statistics and episode-level holdout
as the transition model. For Windy Corridor, the nominal policy is trained
on a random $90\%$ episode split of $\mathcal D^{+}$ and is tabulated over
all grid states after training.

Table~\ref{tab:transition-hparams} summarizes the learned
counterfactual-transition models.

\begin{table}[t]
\centering
\small
\caption{Counterfactual transition-model architectures and training settings.}
\label{tab:transition-hparams}

\setlength{\tabcolsep}{2pt}
\begin{tabularx}{\columnwidth}{
    @{}
    >{\raggedright\arraybackslash}p{0.27\columnwidth}
    >{\centering\arraybackslash}X
    >{\centering\arraybackslash}X
    >{\centering\arraybackslash}X
    @{}
}
\toprule
 & \textbf{Windy Corridor} & \textbf{PointMaze} & \textbf{AntMaze} \\
\midrule
Prediction target
    & next state (75 local offsets)
    & planar displacement (2-D)
    & full-state displacement (29-D) \\

Hidden layers
    & $[128,128]$
    & $[128,128]$ + generator
    & $[512,512]$ \\

Activation
    & ReLU
    & ReLU
    & ReLU \\

Lipschitz control
    & locality restriction, $r=2$
    & target-action spectral projection,
      $L_a\leq1$
    & all-layer spectral normalization,
      coef.\ $2.0$ \\

Power iterations
    & --
    & --
    & 10 \\

Predictive form
    & categorical (tabulated)
    & stochastic, $M=8$
    & deterministic (Dirac) \\

Batch size
    & 256
    & 256
    & 1024 \\

Training updates
    & 6k
    & 2.5k
    & 20k \\

Learning rate
    & $10^{-3}$
    & $3\times10^{-4}$
    & $10^{-3}$ \\

Optimizer
    & Adam
    & Adam
    & Adam \\

Checkpoint selection
    & final iterate
    & earliest validation minimum
    & final iterate \\
\bottomrule
\end{tabularx}
\end{table}

For PointMaze, checkpoint selection uses the earliest minimum of
\[
    \mathcal{L}_{\mathrm{ETT}}(\mathcal{V})
    =
    \mathcal{L}_{\mathrm{diag}}(\mathcal{V})
    +
    \mathcal{L}_{\mathrm{off}}(\mathcal{V}),
\]
evaluated on the held-out validation partition $\mathcal{V}$.
AntMaze and Windy Corridor use the final transition-model iterate without
validation-based checkpoint selection. For Windy Corridor, the factual
next-state network and the counterfactual kernel $\kappa_\eta$ share the
settings in Table~\ref{tab:transition-hparams}, and the background weight
is $\epsilon=10^{-3}$.

The AntMaze counterfactual-replay configuration is summarized in
Table~\ref{tab:rollout-hparams}.

\begin{table}[t]
\centering
\small
\caption{AntMaze counterfactual rollout and replay hyperparameters.}
\label{tab:rollout-hparams}

\setlength{\tabcolsep}{3pt}
\begin{tabularx}{\columnwidth}{
    @{}
    >{\raggedright\arraybackslash}p{0.43\columnwidth}
    >{\raggedright\arraybackslash}X
    @{}
}
\toprule
\textbf{Hyperparameter} & \textbf{Value} \\
\midrule
Rollout start
    & route-dependent split point \\

Rollout horizon
    & 90 \\

Rollouts per refresh
    & 2,000 \\

Refresh interval
    & every training round \\

Generated-data fraction
    & 0.5 \\

State feasibility
    & state bounds and maze free space \\
\bottomrule
\end{tabularx}
\end{table}

The Windy Corridor configuration is summarized in
Table~\ref{tab:windy-rollout-hparams}.

\begin{table}[t]
\centering
\small
\caption{Windy Corridor counterfactual rollout and replay hyperparameters.}
\label{tab:windy-rollout-hparams}

\setlength{\tabcolsep}{3pt}
\begin{tabularx}{\columnwidth}{
    @{}
    >{\raggedright\arraybackslash}p{0.43\columnwidth}
    >{\raggedright\arraybackslash}X
    @{}
}
\toprule
\textbf{Hyperparameter} & \textbf{Value} \\
\midrule
Rollout start
    & recorded $(s,a)$, uniform over $\mathcal D^{+}$ \\

Rollout horizon
    & 400 \\

Rollouts per refresh
    & 10,000 \\

Refresh interval
    & every training round (5k updates) \\

Training rounds
    & 20 \\

Generated-data fraction
    & 1.0 (critic); actor BC uses $\mathcal D^{+}$ \\

State feasibility
    & tabulated over legal grid states \\
\bottomrule
\end{tabularx}
\end{table}

\textbf{Baseline configurations.}
For GCIQL, both continuous-control environments use expectile $0.9$ and
Polyak coefficient $0.005$. The PointMaze actor uses AWR with inverse
temperature $3.0$ and weights clipped at $100$, whereas AntMaze uses the
DDPG+BC-style actor with coefficient $0.3$.

For HIQL, both continuous-control environments use three hidden layers of
size $512$, layer normalization, representation dimension $10$, expectile
$0.7$, low- and high-level actor coefficients of $3.0$, Polyak coefficient
$0.005$, learning rate $3\times10^{-4}$, and batch size $1024$. PointMaze
uses discount $0.95$ and a $10$-step subgoal horizon, whereas AntMaze uses
discount $0.99$ and a $25$-step subgoal horizon. HIQL is trained for
$330{,}000$ updates with checkpoints every $1{,}000$ updates.

The PointMaze subgoal horizon differs from the OGBench~\citep{park2025ogbench} default of $25$.
We evaluated horizons in $\{5,10,25\}$ together with actor coefficients
in $\{3,10\}$; the $5$-step setting produced nearly degenerate subgoal
advantages because most sampled subgoals already lay within the success
radius, and we retained a horizon of $10$.

For Windy Corridor, all baselines are trained on $\mathcal D^{+}$ for
$100{,}000$ updates and evaluated every $250$ updates with the protocol of
Appendix~\ref{app:windy-setup}. Their minibatches are drawn uniformly over
distinct recorded state--action pairs, as for CRL. GCBC uses the CRL actor
trained only with the behavioral-cloning loss on future-state relabeled
goals. GCIQL uses expectile $0.9$, Polyak coefficient $0.005$, and the AWR
actor with inverse temperature $3.0$ and weights clipped at $100$. Its
$Q$ and $V$ networks use one hidden layer of $128$ units on the critic's
one-hot features and are trained with Adam at learning rate
$3\times10^{-4}$. HIQL uses the continuous-control configuration above
with a categorical low-level actor, discount $0.99$, and a $3$-step
subgoal horizon, since Windy Corridor demonstrations are only $10$ to
$40$ steps long.

For both GCIQL and HIQL, goal relabeling uses a binary $1/0$ reward for
the first non-failure arrival at the relabeled goal rather than the
$0/-1$ convention used by the reference OGBench implementations.
Absorbing failure states use a zero bootstrap mask, while time-limit
truncations continue to bootstrap. Other baseline settings follow the
corresponding reference implementations.

\subsection{Result Details}
\label{sec:result-details}

Table~\ref{tab:per-seed} reports the complete per-seed results for all three benchmarks: for each method and
random seed, the success rate averaged over the final $10\%$ of evaluation checkpoints. These are the per-seed
values behind the means reported in Section~\ref{sec:experiments} and plotted in
Figures~\ref{fig:minigrid-result} and~\ref{fig:results}.

\begin{table}[t]
\centering
\footnotesize
\setlength{\tabcolsep}{2.8pt}
\caption{Per-seed final success rates (mean over the final $10\%$ of evaluation checkpoints) on all three benchmarks.}
\label{tab:per-seed}

\resizebox{\columnwidth}{!}{%
\begin{tabular}{@{}llccccc@{}}
\toprule
Benchmark & Method & \multicolumn{5}{c}{Per-seed} \\
\midrule
\multirow{5}{*}{Windy}
 & CauCRL         & 0.797 & 0.984 & 0.919 & 0.895 & 1.000 \\
 & CRL            & 0.513 & 0.513 & 0.513 & 0.513 & 0.513 \\
 & GCIQL          & 0.513 & 0.513 & 0.513 & 0.513 & 0.513 \\
 & GCBC           & 0.586 & 0.562 & 0.635 & 0.574 & 0.574 \\
 & HIQL           & 0.513 & 0.513 & 0.513 & 0.513 & 0.513 \\
\midrule
\multirow{6}{*}{PointMaze}
 & CauCRL         & 1.000 & 1.000 & 1.000 & 1.000 & 1.000 \\
 & CRL            & 0.347 & 0.347 & 0.347 & 0.347 & 0.347 \\
 & GCIQL          & 0.347 & 0.347 & 0.347 & 0.347 & 0.347 \\
 & GCBC           & 0.347 & 0.347 & 0.347 & 0.347 & 0.347 \\
 & HIQL           & 0.194 & 0.173 & 0.166 & 0.213 & 0.135 \\
\midrule
\multirow{5}{*}{AntMaze}
 & CauCRL         & 0.712 & 0.869 & 0.798 & 0.820 & 0.879 \\
 & CRL            & 0.596 & 0.620 & 0.646 & 0.677 & 0.605 \\
 & GCIQL          & 0.532 & 0.612 & 0.516 & 0.913 & 0.508 \\
 & GCBC           & 0.476 & 0.499 & 0.429 & 0.529 & 0.559 \\
 & HIQL           & 0.401 & 0.390 & 0.421 & 0.417 & 0.418 \\
\bottomrule
\end{tabular}%
}
\end{table}

\section{Related Work}
\label{app:related-work}

\textbf{Contrastive goal-conditioned reinforcement learning.}
Contrastive reinforcement learning (CRL) casts goal-conditioned control as learning a critic that distinguishes future achieved states from unrelated states, thereby connecting contrastive representation learning with goal-reaching value functions~\citep{NEURIPS2022_e7663e97}.
Subsequent work has extended this framework along several directions.
Stable CRL studies architectural and optimization choices that make contrastive goal-conditioned policies effective on offline, image-based robotic manipulation tasks~\citep{zheng2024stabilizing}.
TD InfoNCE replaces Monte-Carlo future-state estimation with a temporal-difference contrastive objective, enabling trajectory stitching and more sample-efficient long-horizon prediction~\citep{zheng2024contrastive}.
More recently, JaxGCRL systematically studies modern CRL design choices and finds strong performance with symmetric InfoNCE objectives and distance-based critics~\citep{bortkiewicz2025accelerating}, while quasimetric representations combine contrastive successor representations with temporal-distance structure for offline goal-conditioned learning~\citep{myers2025offline}.

Our approach is complementary to these developments: CauCRL modifies the transition and future-outcome supervision used to train a goal-conditioned critic, rather than relying on a particular contrastive objective.
For example, InfoNCE-based CRL can use counterfactually augmented future states as positive pairs, while a TD-based contrastive method can apply the same robust transition construction to the one-step successors used in its bootstrapped targets.

\end{document}